\documentclass[preprint,11pt]{elsarticle}

\usepackage{amsmath,amssymb,amsxtra,mathtools,amsthm}

\journal{Neurocomputing}

\usepackage{algorithm}
\usepackage{algorithmic}
\usepackage{bm}
\usepackage{dsfont}
\usepackage{enumitem}
\usepackage{hyperref}
\usepackage{mathtools}
\usepackage{graphicx}
\usepackage{caption}
\usepackage{subcaption}
\usepackage{fullpage} 

\usepackage{color} 

\def\ve#1{\mathchoice{\mbox{\boldmath$\displaystyle\bf#1$}}
{\mbox{\boldmath$\textstyle\bf#1$}}
{\mbox{\boldmath$\scriptstyle\bf#1$}}
{\mbox{\boldmath$\scriptscriptstyle\bf#1$}}}

\newcommand{\g}{{\ve g}}
\newcommand{\ind}[1]{\mathds{1}_{\left\{#1\right\}}}
\newcommand{\mymod}[1]{\ (\mathrm{mod}\ #1)}
\newcommand{\norm}[2][]{#1\lVert #2 #1\rVert}
\newcommand{\s}{\textrm{s}}
\newcommand{\w}{{\ve w}}
\newcommand{\x}{{\ve x}}

\newcommand{\z}{{\ve z}}
\newcommand{\A}{\textrm{A}}
\newcommand{\C}{\textrm{C}}
\newcommand{\E}{\mathbb{E}}
\newcommand{\I}{\textrm{I}}
\newcommand{\M}{\mathbf{M}}
\newcommand{\R}{\mathbb{R}}

\DeclareMathOperator*{\argmin}{argmin}
\DeclareMathOperator*{\relu}{{\textrm{ReLU}}}

\renewcommand{\a}{{\ve a}}
\newcommand{\q}{{\ve q}}

\definecolor{firebrick}{RGB}{178,34,34}

\newtheorem{assn}{Assumption}
\newtheorem{dfn}{Definition}
\newtheorem{lemma}{Lemma}
\newtheorem{rmrk}{Remark}
\newtheorem{thm}{Theorem}
\newcommand{\ignore}[1]{}
\begin{document}

\begin{frontmatter}



\title{Depth-2 Neural Networks Under a Data-Poisoning Attack}





\author[inst1]{Sayar Karmakar}

\affiliation[inst1]{organization={Department of Statistics, University of Florida},
            addressline={230 Newell Drive}, 
            city={Gainesville},
            postcode={32611}, 
            state={FL},
            country={USA}}

\author[inst2]{Anirbit Mukherjee}

\affiliation[inst2]{organization={Department of Computer Science, The University of Manchester},
            addressline={Kilburn Building, Oxford Road}, 
            city={Manchester},
            postcode={M13 9PL},
            country={UK}}
            
\author[inst3]{Theodore Papamarkou}

\affiliation[inst3]{organization={Department of Mathematics, The University of Manchester},
            addressline={Alan Turing Building, Oxford Road}, 
            city={Manchester},
            postcode={M13 9PL},
            country={UK}}

\begin{abstract}
In this work, we study the possibility of defending against data-poisoning attacks while training a shallow neural network in a regression setup. We focus on doing supervised learning for a class of depth-2 finite-width neural networks, which includes single-filter convolutional networks. In this class of networks, we attempt to learn the network weights in the presence of a malicious oracle doing stochastic, bounded and additive adversarial distortions on the true output during training. For the non-gradient stochastic algorithm that we construct, we prove worst-case near-optimal trade-offs among the magnitude of the adversarial attack, the weight approximation accuracy, and the confidence achieved by the proposed algorithm. As our algorithm uses mini-batching, we analyze how the mini-batch size affects convergence. We also show how to utilize the scaling of the outer layer weights to counter output-poisoning attacks depending on the probability of attack. Lastly, we give experimental evidence demonstrating how our algorithm outperforms stochastic gradient descent under different input data distributions, including instances of heavy-tailed distributions.
\end{abstract}


\end{frontmatter}

\section{Introduction}

The seminal paper by \citet{szegedy2013intriguing} was among the first to highlight a key vulnerability in state-of-the-art neural network architectures such as GoogLeNet, that adding small imperceptible adversarial noise to test data can dramatically impact the performance of the network. In these cases, despite the vulnerability of the predictive models to the distorted input,  
human observers are still able to correctly classify adversarially corrupted data. 

In the last few years, experiments with adversarially attacked test data have been replicated on several state-of-the-art neural network implementations \citep{goodfellow2014explaining,papernot2017practical,behzadan2017vulnerability,huang2017adversarial}. 
This phenomenon has also resulted in new adversarial defenses being proposed to counter the attacks. Such empirical observations have been systematically reviewed in \citet{akhtar2018threat,qiu2019review}.
On the other hand, the case of data-poisoning or adversarially attacked training data \citep{wang2018data, zhang2019online, koh2018stronger, schwarzschild2021just}  has received much less attention from theoreticians - and in this work, we take some steps towards bridging that gap. 




An optimization formulation of adversarial robustness, in terms of adversarial risk minimization on the test data, has been extensively explored in recent years; 
multiple attack strategies have been systematically catalogued in 
\citet{dou2018mathematical,lin2019nesterov,song2018improving}, computational hardness of finding an adversarial risk minimizing hypothesis has been analyzed in \citet{bubeck2018adversarial,degwekar2019computational,schmidt2018adversarially,montasser2019vc}, 
the issue of certifying adversarial robustness of a given predictor has been analyzed in \citet{raghunathan2018certified,raghunathan2018semidefinite},
and 
bounds on the Rademacher complexity of adversarial risk have been explored in \citet{yin2018rademacher,khim2018adversarial}.
While these previous works have been tuned to classification tasks, we consider the less explored case of adversarial attacks to neural networks used for regression tasks. Towards this end, our first key step is to make a careful choice of the neural network class to work with, as given in Definition \ref{net}. 

For the optimization algorithm, we draw inspiration from the different versions of  iterative stochastic non-gradient algorithms analyzed in the past \citep{rosenblatt1958perceptron, 
Pal1992MultilayerPF,
freund1999large, 
kakade2011efficient, 
klivans2017learning, 
goel2017learning, 
goel2018learning}.
We generalize this class of algorithms in the form of Algorithm \ref{tronalgo}. By allowing for arbitrary weights in the outer layer, we further expand the class of neural networks beyond the ambit of existing results.
Subsequently, we run Algorithm \ref{tronalgo} to train a neural network
in the presence of an adversarial oracle that makes additive bounded perturbations to the true output of the network. Our theory establishes that there is a near-optimal guarantee of recovery of true weights that our algorithm can achieve. 


\subsection{Related work} 

The existing studies of data-poisoning attacks on neural training have mostly focused on the classification setting. The major kinds of data-poisoning attacks on classifiers and attempts at defending against them can be grouped into three categories. Firstly, in backdoor attacks, the adversary
injects strategically manipulated data \citep{gu2017badnets}. Defence mechanisms against backdoor attacks have been proposed in \citet{liu2018fine,tran2018spectral}. Secondly, in clean label attacks, the adversary does not modify the labels of the corrupted data \citep{shafahi2018poison,zhu2019transferable}. Thirdly, in label flip attacks, the adversary changes the labels of a constant fraction of the data \citep{biggio2011support,xiao2018feature,zhao2017efficient}. In the case of Massart noise, a robust half-space learning algorithm was analyzed in \citet{diakonikolas2019distribution}. For more general predictors, the idea of randomized smoothing \citep{cohen2019certified} has recently been extended and empirically shown to be capable of getting classifiers which are pointwise certifiably robust to label flip attacks \citep{rosenfeld2020certified}.

Specifically for the setup of regression, to the best of our knowledge previous guarantees on achieving robustness against data-poisoning attacks have been limited to linear functions. Further, they have either considered corruptions that are limited to a small subset \citep{jagielski2018manipulating, xiao2018feature} of the input space or have made structural assumptions on the corruption. Despite the substantial progress with understanding robust linear regression \citep{xu2008robust,xu2015show,10.5555/2968826.2968855,10.5555/3042817.3043023, 10.1145/3128572.3140447, li2011compressed, Laska_exactsignal, nguyen2011exact}, the corresponding questions have remained open even for simple neural networks.

Theoretical progress in understanding the limitations of training of a neural network under adversarial attacks on test data or training data has been restricted to deep kernel learning, which is associated with asymptotically large networks \citep{gao2019convergence,li2020gradient}. 
Developments along these lines have been made by \citet{pmlr-v139-wang21r}, where the performance of stochastic gradient descent (SGD) is theoretically established when using it to train asymptotically large neural networks to perform classification in the presence of data-poisoning attacks.


Thus, it has been an open challenge to demonstrate an example of provably robust training when (a) the trained neural network is finite, (b) the training algorithm has the common structure of being iterative and stochastic, and (c) the training data is being adversarially attacked. 

In this work, we take a few steps towards this goal by developing a framework inspired
by the causative attack model~\citep{steinhardt2017certified,barreno2010security}.\footnote{The
learning task in the causative attack model is framed as a game between a defender who seeks to learn and an attacker who aims to prevent this. In a typical scenario, the defender draws a finite number of samples from the true distribution (say $S_c$) and for some $\epsilon \in (0,1)$ the attacker mixes into the training data a set $S_p$ of (maybe adaptively) corrupted training samples such that $\vert S_p\vert = \epsilon \vert S_c\vert$. Now the defender has to train on the set $S_p \cup S_c$. We note that our model is not the same as the causative attack model because we allow for an arbitrary fraction (including all) of the training data to be corrupted in an online fashion by the bounded  additive adversarial attack on the true output.}

\subsection{Summary of results and outline} 

The model of data-poisoning that we consider
includes an additive distortion of the true output.
For every data point,
the additive distortion
is a sample from a possibly different distribution.
On the other hand, a typical model with additive noise
relies on the assumption that all additive distortions
are sampled from a single distribution~\citep{goel2018learning}.
Thus, the assumption of varying distribution of distortions
is what makes a data-poisoning model distinct from
a typical model of noisy output based on additive noise.

Our adversarial neural network herein is
a multi-gate generalization of the single-gate model in \cite{KARMAKAR2022264}.
Furthermore, we allow adversarial attacks to the multi-gate model
that are capable of preventing optimal learning,
as evidenced by our lower bound on the achievable accuracy of recovering
the model parameters from the intact data (Section \ref{opti:tron}). 

The main result (Theorem \ref{main2}) shows that our proposed Algorithm \ref{tronalgo} achieves
a trade-off between accuracy, confidence and maximum allowed adversarial perturbation
while learning the neural network parameters.
This trade-off provides a performance guarantee
that holds (i) in the presence of finitely many gates,
(ii) under an online adversarial attack that does not assume access
to the whole training data at the start of training,
and (iii) for any probability of attack.
To the best of our knowledge, there does not exist in the literature of neural network training any other
such performance guarantee with all these three conditions being simultaneously true.

The parameter recovery accuracy of Algorithm \ref{tronalgo},
as guaranteed in Theorem \ref{main2},
has two salient features.
Firstly, our defense can defeat an adversarial attack
in some scenarios,
in the sense that the risk of the learnt predictor
can be lower than the maximum adversarial distortion (Section \ref{sec:risk}).
Secondly, the accuracy of parameter recovery improves
by upscaling the weights of the second layer (Section \ref{ssc:highq}).


In Section \ref{sec:sim},
we provide empirical evidence that our Algorithm \ref{tronalgo}
attains higher parameter recovery accuracy and faster rate of convergence than SGD.
While a proof of this claim remains an open research question,
our simulation-based observations are consistent
across different input data distributions, probabilities of adversarial attack
and magnitudes of adversarial attack. We also draw the attention of the reader to the experiment in Figure \ref{fig:tron_vs_sgd_tdist} showing that Neuro-Tron outperforms SGD even when the data is sampled from the Student's $\mbox{t}(\nu=4)$ distribution, which does not have enough number of finite moments to be covered by the assumptions of the main Theorem~\ref{main2}. This experiment in particular strongly motivates exciting directions of future research.

\paragraph{Outline of the paper}
In Section \ref{sec:setup}, we give the mathematical setup of the neural networks, distributions and data-poisoning adversary that we use, and define our learning algorithm (Algorithm \ref{tronalgo}).
In Section \ref{sec:thm},
we state our main result,
which pertains to the parameter recovery accuracy of Algorithm~\ref{tronalgo},
as Theorem \ref{main2}.
The proof of Theorem \ref{main2} is given in Appendix \ref{sec:proofthm1},
while appendices~\ref{sec:proofthm1_2}-\ref{app:rec}
contain several lemmas that are needed in the proof.
In Section \ref{opti:tron}, we explain that the accuracy-confidence-attack trade-off we obtain is nearly optimal in the worst-case.
In Section~\ref{ssc:highq}, we propose a way of improving
the accuracy of parameter recovery for Algorithm \ref{tronalgo}
by upscaling the outer layer weights.
In Section \ref{sec:sim}, we perform a simulation study of Algorithm \ref{tronalgo},
among else comparing it with SGD.
We conclude in Section \ref{sec:conc} by motivating relevant directions of future research.



\section{Mathematical setup}
\label{sec:setup}

In a supervised learning setting,
consider observed data pairs
$\z = (\x,y) \in \mathcal{Z} := \mathcal{X}\times \mathcal{Y}$,
where $\x$ is the input in a measure space $\mathcal{X}$ and
$y$ is the output in a measure space $\mathcal{Y}$.
Let $\mathcal{D}$ be the distribution of $\z$ over $\mathcal{Z}$
and let $\mathcal{D}_{\x}$ be the marginal distribution of the input $x$ over $\mathcal{X}$.
Consider a loss function $\ell : {\cal H} \times {\cal Z} \rightarrow \R^+$,
where ${\cal H}$ is the hypothesis space for a learning task performed by a neural network.
We model an adversarial oracle as a map $\bm{O}_{\bm A} : \mathcal{Z} \mapsto \mathcal{Z}$ that corrupts the data $\z \in \mathcal{Z}$ with the intention of impeding the learning task.
The uncorrupted data $\z \sim \mathcal{D}$ are not observed, so they are treated as a latent variable.
Only the corrupted data
$\bm{O}_{\bm A} (\z)$
are observed.
The aim is to find a hypothesis $h \in \mathcal{H}$ that minimizes the true risk
$\mathcal{R}(h; \mathcal{D}) := \E_{\z \sim {\cal D}} \left [ \ell(h,z) \right ]$.

In this work, we consider the case of
input space $\mathcal{X} = \R^n$,
output space $\mathcal{Y} = \R$,
square loss function $l$
and hypothesis space $\mathcal{H}$
of the class $\mathcal{F}_{k,\alpha,\mathcal{A},\mathcal{W}}$
of depth-2 width-k neural networks.
The class $\mathcal{F}_{k,\alpha,\mathcal{A},\mathcal{W}}$
is specified in Definition \ref{net}.

\begin{dfn}[Single-filter neural networks of depth $2$ and width $k$]
\label{net} 
Given a set of $k$ sensing matrices $\mathcal{A} = \{\A_i \in \R^{r\times n} \mid i = 1,\ldots,k \}$, an $\alpha$-leaky $\relu$ activation mapping
$ \sigma(y) =  y{\bf 1}_{y \geq 0} + \alpha y{\bf 1}_{y <0}$
and a filter space ${\cal W} \subseteq \R^r$, we define the function class $\mathcal{F}_{k,\alpha,\mathcal{A}, {\cal W}}$ as
\begin{align*} 
 \mathcal{F}_{k,\alpha,\mathcal{A},{\cal W}}  :=
 \left \{f_{\w} : \R^n \rightarrow \R~\mbox{with}~
 f_{\w} (\x) = \frac{1}{k} \sum_{i=1}^k \sigma \left ( \w^\top \A_i\x \right ) ~|~ \w \in \mathcal{W} \right \}.
\end{align*}
\end{dfn}


Note that the above class of neural networks encompasses the following common instances; (a) single $\relu$ gates as $\mathcal{F}_{1,0,\{ {\rm I}_{n \times n}\},\R^n }$; (b) depth-$2$ width-$k$ convolutional neural networks, when each sensing matrix $\A_i$ has exactly one $1$ in each row, at most one $1$ in each column, and $0$ in all other entries. 

      


For each data point $(\x, y)$,
we assume that
$\exists ~\w^* \in \R^r$ such that
$y = f_{\w^*} (\x)$.
The oracle decides whether to attack a given data point $(\x, y)$
by performing a Bernoulli trial $\alpha_{\x}\sim\mbox{Bernoulli}(\beta (\x))$
with probability of attack $\beta (\x) :=\mbox{Pr} (\alpha_{\x}=1)$.
If $\alpha_{\x}=1$,
then the oracle replaces the original output $y = f_{\w^*}(\x)$ with
$f_{\w^*}(\x) + \xi_\x,$
where $\vert \xi_\x \vert  \leq \theta$ for some fixed $\theta$.
If $\alpha_{\x}=0$,
then the oracle returns the original output $y$ without perturbing it.
To ease notation,
the adversarial output is denoted by $v:=f_{\w^*}(\x) + \alpha_{\x}\xi_\x$.
In summary, the adversarial action is
${\bm O}_{\bm A} (\x, y) =
{\bm O}_{\bm A} (\x, f_{\w^*}(\x)) =
(\x, f_{\w^*}(\x) + \alpha_{\x}\xi_\x ) = (\x, v)$.
The risk minimization problem is to find
$\argmin_{\w \in \mathcal{W}} 
\E_{\x \sim {\cal D}_{\rm \x}} 
[ \left ( f_{\w^*}(\x) - f_\w(\x) \right )^2]
$.

The adversarial oracle designs the output distortion
$f_{\w^*}(\x) + \xi_\x$
under Assumptions \ref{assn:par_sym} and \ref{assn:fin_moments}
for the distribution ${\cal D}_{\rm x}$ of input $\x$.

\begin{assn}[Parity symmetry]
\label{assn:par_sym}
We assume that the distribution $\mathcal{D}_{{\rm \x}}$ of the input $\x$
is symmetric under the parity transformation, i.e if $\x$ is a random variable such that $\x \sim {\cal D}_{\rm \x}$ and $-\x \sim {\cal D}_{\rm \x}$.
\end{assn}

\begin{assn}[Finiteness of first four moments of input norm]
\label{assn:fin_moments}
We assume that the following expectations are finite,
\[ {\rm m}_i \coloneqq 
\E_{\x } \left[ \norm{\x}^i \right],~i =1,2,3,4, \] 
\end{assn}
where $\norm{\cdot}$ denotes the Euclidean norm thereafter. We note that for a measurable function $\beta : \R^n \rightarrow [0,1]$,
Assumption \ref{assn:fin_moments} implies finiteness of the expectations
\[ \beta_i \coloneqq 
\E_{\x } \left[ \beta(\x)\norm{\x}^{i} \right],~i =1,2,3,4. \]
$\beta(\x)$ induces bias in the coin that the adversarial oracle tosses
to decide whether or not to attack the true
output $y$ associated with input $\x$. To ease exposition, we introduce the notation
\begin{gather*}
\bar{\A} := \frac{1}{k}\sum_{i=1}^k \A_i,
~\Sigma := \mathrm{E}\left[\x\x^\top\right],\\
\lambda_1 := \lambda_{\min}\left(\bar{\A}\Sigma \M^T\right),
~\lambda_2 := \sqrt{\lambda_{\max}\left(\M^T \M\right)},
~\lambda_3 := \frac{1}{k}\sum_{i=1}^k \lambda_{\max}\left( {\A_i\A_i^\top} \right),
\end{gather*}
where $\lambda_{\min}$ and $\lambda_{\max}$ denote the minimum and maximum eigenvalue,
respectively,
and $\M \in \R^{r \times n}$ is a ``sensing matrix''.

Algorithm \ref{tronalgo} summarizes
the proposed neural network training procedure
in the presence of the assumed adversarial oracle.
We refer to Algorithm \ref{tronalgo} as the Neuro-Tron.
Neuro-Tron is a stochastic optimization algorithm that does not use gradients,
and it is inspired by 
\citet{kakade2011efficient, klivans2017learning, goel2018learning}.


\begin{algorithm}[tb] 
\caption{Neuro-Tron (mini-batched, multi-gate, single-filter, stochastic algorithm)}
\label{tronalgo}
\begin{algorithmic}[1]
\STATE {\bf Input:} Sampling access to the marginal input distribution ${\cal D}_{\x}$ 
\STATE {\bf Input:} Access to adversarial output $v\in\R$
for any input $\x\in\R^n$
\STATE {\bf Input:} Access to output  $f_{\w} (\x)$ of any $f_\w \in  \mathcal{F}_{k,\alpha,\mathcal{A},{\cal W}}$
for any $\w\in \R^r$ and any input $\x$
\STATE {\bf Input:} A sensing matrix $\M \in \R^{r \times n}$
\STATE {\bf Input:} A starting point $\w_1$
\STATE {\bf Input:} Number $T$ of batches
\STATE {\bf Input:} Batch size $b$
\STATE {\bf Input:} Learning rate $\eta$
\item[]
\FOR{$t = 1,\ldots$,\,{\rm T}}
    \STATE Sample batch
    $s_t \coloneqq
    (\x_{t_1},\ldots,\x_{t_b})$,
    where $\x_{t_i}\sim {\cal D}_{\x},~i=1,\ldots,b$ 
    \FOR{$i = 1,\ldots,b$}
    \STATE  The oracle samples
    $\alpha_{t_i} \in \{0,1\}$ with probability
    $\{ 1 -\beta(x_{t_i}), \beta(x_{t_i})\}$
   \STATE The oracle replies with
   $v_{t_i} := f_{\w^*}(\x_{t_i} )+\alpha_{t_i}\xi_{t_i}$
  \ENDFOR
  \STATE Form the so-called Tron-gradient, 
    \[\g^{(t)} := \M \left ( \frac{1}{b} \sum_{i=1}^b
    \left ( \Big (v_{t_i} - f_{\w^{(t)}}(\x_{t_i}) \Big) \x_{t_i} \right ) \right )\]
  \STATE  $\w^{(t+1)} = \w^{(t)} + \eta \g^{(t)}$
\ENDFOR  
\end{algorithmic}
\end{algorithm}


\section{Probabilistic performance bounds}\label{sec:thm} 

This section states the main theorem (Theorem \ref{main2}),
which provides the optimal learning rate
for Algorithm \ref{tronalgo}
to approximate the neural network weights
at a given accuracy level $\epsilon$
with a given failure probability $\delta$.
To illustrate the consequences of Theorem \ref{main2},
the special case of normally distributed input data is discussed.

\subsection{The main theorem}

\begin{thm}[Trade-off between accuracy and failure]
\label{main2} 
Let $f_{\w}$ be a neural network belonging to
the function class $\mathcal{F}_{k,\alpha,\mathcal{A}, {\cal W}}$ of Definition \ref{net}.
For each data point $(\x, y)$ of input $\x$ and output $y$,
assume that
$\exists ~\w^* \in \R^r$ such that
$y = f_{\w^*} (\x)$.
Consider an adversarial oracle acting as
${\bm O}_{\bm A} (\x, y) = (\x, f_{\w^*}(\x) + \xi_\x)$,
where $\vert \xi_\x \vert  \leq \theta$ for some fixed $\theta$. Additionally, assume that Assumptions \ref{assn:par_sym} and \ref{assn:fin_moments} are satisfied, and that $\lambda_1 >0$.
\begin{enumerate}
\item
Assume that $\theta = 0$, which corresponds to the case of
uncorrupted adversarial action
${\bm O}_{\bm A} (\x, y) = (\x, f_{\w^*}(\x))$.
If the learning rate $\eta$ in Algorithm \ref{tronalgo} is equal to
\[
\eta = \frac{
\lambda_1}{\gamma(1+\alpha)\lambda^2_2\lambda_3(\mathrm{m}_4/b+\mathrm{m}_2^2(1-1/b))},
\]
where $\gamma > \max \left\{ C,1\right\}$ with $C:=\lambda_1^2 (\lambda_2^2\lambda_3 (\mathrm{m}_4/b+\mathrm{m}_2^2(1-1/b)))^{-1}$,
then for accuracy $\epsilon \geq 0$,
for failure probability $\delta \geq 0$,
and for ${\rm T} = {\large \mathcal{O}}\Big(\log\left(\frac{\norm{\w^{(1)}-\w^*}^2}{\epsilon^2\delta}\right)\Big)$
it holds that
$\norm{\w^{({\rm T})}-\w^*} \leq \epsilon$ with probability at least $1-\delta$.

\item
Assume that 
$\theta \in (0, \theta_*)$ for some $\theta_* >0$,
which corresponds to the case of adversarial perturbation via additive noise.
We assume that $\beta_1$ is such that the constant ${\rm c}_{\rm trade-off} := \frac{ (1+\alpha)\lambda_1}{\beta_1\lambda_2}-1 >0$ 
Moreover, assume that the distribution ${\cal D}_{\x}$, matrix $\M$ in Algorithm \ref{tronalgo}, noise bound $\theta_*$,
target accuracy $\epsilon$ and target confidence $\delta > 0$ are such that
\begin{align}\label{condmain1}
\theta_*^2 = \epsilon^2 \delta {\rm c}_{\rm trade-off},~
\epsilon^2\delta  <  \norm{\w^{(1)}-\w^*}^2 . 
\end{align}
If the learning rate $\eta$ in Algorithm \ref{tronalgo} is equal to
\begin{equation*} 
\eta =  \frac{{\beta}_1 {\rm c}_{\rm trade-off}}{\gamma(1+\alpha)^2 \lambda_2 \lambda_3 \Big ( (\beta_1{\rm m}_2+{\rm m}_2^2)\left(1-\frac{1}{b}\right) +\frac{\beta_3+{\rm m}_4}{b} \Big )},
\end{equation*}
where
\begin{equation*}
\gamma > \max \left\{
\frac{\left(\beta_1 c_{\rm trade-off  }\right)^2}
{(1+\alpha)^2\lambda_3\left(\left(\beta_1{\rm m}_2+
{\rm m}_2^2\right)\left(1-\frac{1}{b}\right) +\frac{\beta_3+{\rm m}_4}{b}\right)},\,
C_2 \right\} > 1
\end{equation*}
with
\begin{equation*}
C_2:= \frac{\epsilon^2\delta + \frac{\theta^2\left((\beta_1^2+\beta_1{\rm m}_2)\left(1-\frac{1}{b}\right) +\frac{\beta_2+\beta_3}{b}\right)}{(1+\alpha)^2\lambda_3\left(\beta_1{\rm m}_2+{\rm m}_2^2\right)\left(1-\frac{1}{b}\right) +\frac{\beta_3+{\rm m}_4}{b}}}{\epsilon^2\delta - \frac{\theta^2}{{\rm c}_{\rm trade-off}}},
\end{equation*}
then for
\begin{equation*}
{\rm T}  = {\large \mathcal{O}} \left( \log \left[ \frac{\norm{\w^{(1)}-\w^*}^2}{\epsilon^2\delta - \frac{\theta^2}{{\rm c}_{\rm rate} }} \right]\right)
\end{equation*}
with
\begin{equation*}
{\rm c}_{\rm rate} := \frac{\gamma - 1}{\frac{\mathrm{m}_2+\mathrm{m}_3}{(1+\alpha)^2\lambda_3(\mathrm{m}_3 + \mathrm{m}_4)} + \frac{\gamma}{{\rm c}_{\rm trade-off}} }
\end{equation*}
it holds that
$\norm{\w^{({\rm T})}-\w^*} \leq \epsilon$ with probability at least $1-\delta$.
\end{enumerate}
\end{thm}


\begin{rmrk}
\label{theorem1_remarks}
Some remarks on Theorem \ref{main2} follow.
\begin{enumerate}
\item Theorem \ref{main2} places weak conditions, which can be easily met in practice.
Firstly, it is easy to find a distribution ${\cal D}_{\x}$
that satisfies Assumptions \ref{assn:par_sym} and \ref{assn:fin_moments}
and that has a positive definite covariance matrix $\Sigma$.
Secondly, for any full rank $\M \in \R^{r \times n}$,
any matrix $\C \in \R^{r \times n}$
and any even width $w$ (say $w=2k$),
the sensing matrices in Definition \ref{net} can be set to
${\cal A} = \{ (\M - k\C, \M - (k-1)\C,...,\M - \C, \M+\C,...,\M + k\C \}$.
Then $\bar{A} = \M$ has full rank, so $\lambda_1 = \lambda_{\min}(\M \Sigma \M^\top) >0$
as required.
Thirdly, it is easy to construct a sampling scheme
that generates a matrix $\M \in \R^{r \times n}$ with $1 \leq r \leq n$
which is full rank with high probability. To this end,
generate independent $\textbf{g}_i=(g_i^1,\dots,g_i^r)^{\top}\sim N(0,\I_{r \times r})$ and construct $G=\sum_{i=1}^k \textbf{g}_i \textbf{g}_i^\top$.
Then $G$ follows a Wishart distribution $\mathbb{W}(\I,k)$ with $k$ degrees of freedom.
Since $\I$ is invertible, $G$ has full rank with probability 1 as long as $k \geq r$.
$G$ can be used as a sub-matrix
to complete it as a matrix $\M \in \R^{r \times n}$ which also has full rank with probability $1$.
Lastly, 
consider the case when $\beta(\x)$ is a constant $\beta$. Then we note that the condition of
$c_{\rm trade-off}$ being positive
is equivalent to $\beta < \frac{(1+\alpha)\lambda_{\min}(\bar{\A} \Sigma \M^\top)}{ \E [ \norm{\x}]\cdot \norm{\M}_2}$.
From here we can see that if we anticipate a large probability $\beta$ of attack,
we can scale the vectors in the support of distribution ${\cal D}_{\x}$
by an appropriate positive factor and enlarge the upper bound for $\beta$ by the same factor. 
\item The uniqueness of the global minimum for $\theta =0$ can be proven by contradiction.
Assume that there are two distinct minima
$\argmin_{\w \in \mathcal{W}} \E_{\x \sim {\cal D}_{\x}}
[ \left ( f_{\w^*}(\x) - f_\w(\x) \right )^2]$.
The application of Theorem \ref{main2} to each minimum separately
implies that Algorithm \ref{tronalgo} gets
arbitrarily close to each minimum, which is a contradiction.
\item In both cases
$\theta = 0$ and $\theta \in (0, \theta_*)$,
the learning rate $\eta$ is an increasing function of
the mini-batch size $b$.
So increasing $b$ increases the rate of convergence.
\item In the case of $\theta \in (0, \theta_*)$,
the term $\epsilon^2\delta - \frac{\theta^2}{c_{{\rm rate}}}$
in the expression of ${\rm T}$ is positive
because of the lower bound imposed on the parameter $\gamma$.
\item In Subsection \ref{opti:tron},
we show that the trade-off between optimization accuracy and failure probability
is near-optimal in the worst-case scenario, that is
when the adversary attacks every data point.
\end{enumerate}
\end{rmrk} 


\subsection{Sketch of the proof of the main theorem}
\label{sec:proofsketch}

The proof of Theorem \ref{main2} is given in
Appendices \ref{sec:proofthm1}, \ref{sec:proofthm1_2}, \ref{app:lem} and \ref{app:rec}.
An outline of the proof follows.
Initially,
the proof disentangles the dependencies between
random variable $\g^{(t)}$,
sampled data $s_t$ and
coin flips 
$a_t:=(\alpha_{t_1},\ldots,\alpha_{t_b})$
that determine whether or not to attack the corresponding output data.
Let $s_{1:t} := (s_1,\ldots, s_t)$
be the training data sampled by Algorithm \ref{tronalgo}
till the $t$-th iteration.
The neural network weights $\w_t$ at time $t$
are determined conditional on $s_{1:(t-1)}$.
The random variable $g_t$ is dependent on
$\s_t$, on $\alpha_{t}$, and on $(\xi_{t_1},\ldots,\xi_{t_b})$.
The key idea in the proof is to find a tight upper bound on the random variables
\begin{equation*}
\E_{s_t,\alpha_t} \left[ \norm{\w^{(t+1)} - \w^*}^2 - \norm{\w^{(t)} - \w^*}^2
\,|\, s_{1:(t-1)} \right].
\end{equation*}
To acquire such an upper bound,
we invoke
Assumption \ref{assn:par_sym}
and we track the combinatorial effect of mini-batching.
Finally, we take total expectations over the above upper bound
and reduce the problem of finding convergence times
to a problem of analyzing certain algebraic recursions.
These recursions are established in the lemmas of Appendix \ref{app:rec}.


\subsection{Performance bounds for normally distributed input data}

To understand the constraints imposed by Equation \eqref{condmain1}
of Theorem \ref{main2},
we assume normally distributed input data
and consider a single $\relu$ gate neural network
$\mathcal{F}_{1,0,\{ \mathbf{I}_{n\times n} \},\mathbb{R}^n}$.
Lemma \ref{lm1} provides the constant
${\rm c}_{\rm trade-off}$ under this setting.

\begin{lemma}\label{lm1}
[Accuracy-failure trade-off for normally distributed input]
Consider a single $\relu$ gate neural network
$\mathcal{F}_{1,0,\{ \mathbf{I}_{n\times n} \},\mathbb{R}^n}$.
Assume that the input data $\x$ are normally distributed according to
$\mathcal{D}_{\x} = \mathcal{N}(0,\sigma^2\mathbf{I}_{n\times n})$.
If $\beta(\x)=:\beta\in (0, 1)$ for all $\x$ and
$\M=\mathbf{I}_{n \times n}$ in Algorithm~\ref{tronalgo},
then the constant $c_{\rm trade-off}$ in Theorem~\ref{main2} is given by
\begin{equation}\label{gauss}
    {\rm c}_{\rm trade-off}=
    \frac{\theta_*^2}{\epsilon^2\delta} = \frac{\sigma}{\sqrt{2}\beta} \frac{\Gamma\left(\frac{n}{2}\right)}{\Gamma\left(\frac{n+1}{2}\right)} -1.
\end{equation}
\end{lemma}

\noindent The proof of Lemma~\ref{lm1} can be found in Appendix \ref{lem:gauss}. If the input data distribution is
${\cal N}(0,\sigma^2 {\bf I})$,
where $\sigma^2$ is an increasing function of the input data dimension $n$
such that the right-hand side of Equation \eqref{gauss} remaines fixed,
then Equation~\eqref{gauss} provides a sufficient condition to
defend against an adversary with a fixed corruption budget of $\theta_*$,
with a desired accuracy of $\epsilon$ and with failure probability of $\delta$.


\subsection{Understanding the prediction risk}\label{sec:risk}

The prediction risk of a neural network
$\mathcal{F}_{k,\alpha,\mathcal{A},{\cal W}}$
at time $\rm T$ is
\begin{equation*}
\E_{\x \sim {\cal D}_{\x}}
\left[ \left ( f_{\w^*}(\x) - f_{\w^{(\rm T)}}(\x) \right )^2\right]
=\E_{\x \sim {\cal D}_{\rm \x}}
\left[ \left ( \frac{1}{k} \sum_{i=1}^k
\left\{ \sigma \left( \w^{*\top} \A_i\x \right) -
\sigma \left( \w^{(\rm T)\top} \A_i\x \right) \right\} \right )^2\right] .
\end{equation*} 
As shown in Lemma \ref{lem:diff_f_sq} (Appendix \ref{app:lem}),
if the conditions of Equation \eqref{condmain1} of Theorem \ref{main2} are satisfied
and if the upper bound
\begin{equation}
\label{ineq:bound}
\left ( \frac{(1+\alpha)^2}{k \delta 
{\rm c}_{\rm trade-off}}
\sum_{i=1}^k \lambda_{\max}(\A_i\A_i^\top) \right ) 
\E_{\x \sim {\cal D}_{\x}}
\left[ \|\x\|^2 \right]<1
\end{equation}
holds at iteration $\rm T$,
then the risk
$\E_{\x \sim {\cal D}_{\x}}
[ \left ( f_{\w^*}(\x) - f_{\w^{(\rm T)}}(\x) \right )^2 ]$
is bounded above by $\theta_*^2$.





It is easy to demonstrate cases for which Inequality~\eqref{ineq:bound} holds.
For example, the assumptions of Lemma \ref{lm1}
imply that Inequality~\eqref{ineq:bound} is equivalent to
\begin{equation}
\label{eq:b1_bound}
n=\E_{\x \sim {\cal D}_{\x}}  [\|x\|^2] \leq \delta\left(\frac{1}{\beta_1} -1\right)
\text{ or } \beta_1 \leq \frac{1}{1+n/\delta}
\end{equation}
in the case of a single $\relu$ gate neural network
$\mathcal{F}_{1,0,\{ \mathbf{I}_{n\times n} \},\mathbb{R}^n}$
trained on normally distributed input data.
Equation~\eqref{eq:beta1} and Inequality~\eqref{eq:b1_bound} yield the upper bound
\begin{equation}
\label{eq:b_bound}
\beta<\frac{1}{\sqrt{2}}
\left[\frac{\Gamma\left(\frac{n}{2}\right)}{\Gamma\left(\frac{n+1}{2}\right)}\right]\frac{1}{1+n/\delta}
\end{equation}
for the probability $\beta$ of adversarial attack.
Note that this bound on $\beta$ depends on the input data dimension $n$.
So, if the probability $\beta$ of attack admits the upper bound of Inequality~\eqref{eq:b_bound}
while training a single $\relu$ gate neural network
on normally distributed input data,
then the learnt weights attain higher average prediction accuracy
than the worst distortion the oracle could have made to any particular output data point.



\subsection{Demonstrating near-optimality in the worst case
}\label{opti:tron}

We recall that Case 1 of Theorem \ref{main2} ($\theta = 0$)
shows that Algorithm \ref{tronalgo} recovers the true filter $\w^*\in\R^r$ when it has access to the uncorrupted data. 
Consider a filter value $\w_{{\x}} \neq \w^*$
given the true filter $\w^*$,
and suppose that $\theta_* = \zeta$ for some $\zeta  \geq \sup_{\x \in \text{supp}(\mathcal{D}_{\x})} \vert f_{\w_{\rm adv}}(x) - f_{\w^*}(x) \vert$,
where $\text{supp} (\mathcal{D}_{\x})$ denotes the support of $\mathcal{D}_{\x}$.
Assume that ${\cal D}_{\x}$ is compactly supported,
so that the supremum exits.
In this setting,
Equation~\eqref{condmain1} yields
$\epsilon^2 \geq \zeta^2 / c_{\rm trade-off}$. Hence proving optimality of the guarantee
is equivalent to showing the existence of an attack which satisfies the
upper bound $\zeta$ of
$\sup_{\x \in \text{supp}(\mathcal{D}_{\x})}\vert f_{\w_{\rm adv}}(x) - f_{\w^*}(x) \vert$
and
for which the best possible accuracy nearly saturates the lower bound
$\zeta^2 / c_{\rm trade-off}$ of $\epsilon^2$.

If the adversarial oracle ${\bf O_{\A}}$ is queried at $\x$
under this choice of $\theta_*$, then
the oracle replies with
$\xi_\x + f_{\w_*}(\x)$,
where $\xi_\x = f_{\w_{\rm adv}}(\x) - f_{\w_*}(\x)$.
So, the data Algorithm \ref{tronalgo} receives
are exactly realized with filter $\w_{\rm adv}$.
Thus, Case 1 of Theorem \ref{main2} implies that
Algorithm \ref{tronalgo}
converges to $\w_{\rm adv}$
with high probability
and with error
$\norm{\w_{\rm adv} - \w_*} \le \epsilon$. 

We now consider the above attack happening 
to a single $\relu$ gate neural network
$f_{\w_*} (\x) = \relu(\w_*^\top \x),~\x\in\R^n,$
with $\zeta = r\norm{\w_{\rm adv} - \w_*}$,
where $r: = \sup_{\x \in {\rm supp} ({\cal D}_{\x})} \norm{\x}$.
Assume that $r$ is finite and that
${\cal D}_{\x}$ satisfies Assumptions
~\ref{assn:par_sym} and~\ref{assn:fin_moments}.
This choice of $\zeta$ is valid since the following holds, 
\begin{equation*} 
\sup_{\x \in \text{supp}(\mathcal{D}_{\x})} \vert \relu(\w_{\rm adv}^\top \x) - \relu(\w_*^\top \x) \vert 
\leq r\norm{\w_{\rm adv} - \w_*} = \zeta .
\end{equation*} 
Such a setup for training a single $\relu$ gate neural network
on output data additively corrupted by at most $\zeta = r\norm{\w_{\rm adv} - \w_*}$
demonstrates a worst case scenario (i.e $\beta(\x) =1$) in which
the accuracy guarantee of $\epsilon^2 \geq \zeta^2 / c_{\rm trade-off}$
is optimal up to a constant $r^2/c_{\rm trade-off}$.
The near-optimality of
Equation~\eqref{condmain1}
holds for any algorithm defending against this attack,
if the algorithm has the property of recovering the parameters correctly
when the output data are exactly realizable.  

\subsection{Defense against data-poisoning attacks via upscaled outer layer weights}\label{ssc:highq}

Definition~\ref{net2} generalizes the class of neural networks of Definition \ref{net} by introducing a weighted sum of gates computed by a neural network in the class.
The weights $q\in{\cal W}_2\subseteq\R^k$ of the second layer
play the role of weights in the sum of gates and
augment the network parameter space ${\cal W}_1$
of Definition~\ref{net} by ${\cal W}_2$.

~\\
\begin{dfn}[Weighted neural networks of depth $2$ and width $k$]
\label{net2}
Given $k$ sensing matrices $\mathcal{A} = \{\A_i \in \R^{r\times n} \mid i = 1,\ldots,k \}$, an $\alpha$-leaky $\relu$ activation mapping $\sigma(y) =  y{\bf 1}_{y \geq 0} + \alpha y{\bf 1}_{y <0}$, a filter space ${\cal W}_1 \subseteq \R^r$ and a space of values for the second layer weights ${\cal W}_2 \subseteq \R^k$, we define the function class $\mathcal{F}_{k,\alpha,\mathcal{A}, {\cal W}_{1,2}}$ as
\begin{equation*} 
 \mathcal{F}_{k,\alpha,\mathcal{A},{\cal W}_{1,2}}  :=
 \left\{f_{\q,\w} : \R^n \rightarrow \R,~
 f_{\q,\w}(\x)=\frac{1}{k} \sum_{i=1}^k q_i \cdot \sigma \Big ( \w^\top \A_i\x \Big )  \in \R ~|~ \w \in \mathcal{W}_1, ~\q \in {\cal W}_2 \right\} .
\end{equation*}
\end{dfn}

The analysis in the proof of
Theorem \ref{main2}
is applicable when $f_{\w}$ is replaced by $f_{\q,\w}$ for fixed $\q$.
Consequently, Theorem \ref{main2} continues to hold
for $\bar{\A}$ and $\lambda_3$ set to
\begin{gather}
\label{abar_qi}
\bar{\A} := \frac{1}{k}\sum_{i=1}^k q_i \A_i,
~\lambda_3 := \frac{1}{k}\sum_{i=1}^k q_i^2  \lambda_{\max}\left( {\A_i\A_i^\top} \right).
\end{gather}
Equation~\eqref{abar_qi} sets the constraint of positive
$\lambda_1 = \lambda_{\min}\left(\bar{\A}\Sigma \M^T\right)$
on $\M$.

Here we consider a special case of a neural network with a weighted sum of gates,
satisfying Definition~\ref{net2}.
By analyzing this special case we shall reveal an interesting insight about how weights in the outer layer of the network can help to defend against the attack being considered. 
Consider neural networks $f_{\bm{1},\w}$ in
$\mathcal{F}_{k,\alpha,\mathcal{A},{\cal W}_{1,2}}$
with sensing matrices $\A_i,~ i= 1,\ldots, k,$
for the first layer of the network and note that
$f_{\bm{1},\w}=f_{\w}
\in\mathcal{F}_{k,\alpha,\mathcal{A},{\cal W}}.$
Set $\M$ such that
$\lambda_1 = \lambda_{\min}\left(\bar{\A}\Sigma \M^T\right) >0,$
where $\bar{\A} = \frac{1}{k} \sum_{i=1}^k \A_i$.
Further, given a real number $q\ne 0,$
consider another class of neural networks $f'_{q^2\bm{1},\w}$ in
$\mathcal{F}_{k,\alpha,\mathcal{A},{\cal W}_{1,2}}$.
Note that 
$\lambda_1' \coloneqq \lambda_{\min}(\bar{\A}'\Sigma \M^T) = q^2 \lambda_1 > 0,$
where $\bar{A}' \coloneqq \frac{q^2}{k} \sum_{i=1}^k \A_i$.
Thus, $\M$ ensures convergence of Algorithm \ref{tronalgo}
while training over both network classes.
Assume that the constants $\beta_1$ and $\theta_*$,
which characterize the adversarial attack, and
the `lack of confidence' $\delta$ are fixed.
If $\epsilon$ and $\epsilon'$
are the guaranteed accuracies of recovering the true weights
when $\q=\bm{1}$ and $\q=q^2\bm{1}$, respectively,
the it follows from Equation~\eqref{condmain1} that
\begin{equation*}
\epsilon =\theta_*\sqrt{\frac{1}{\delta (\frac{(1+\alpha)\lambda_1}{\beta_1\lambda_2}-1)}},~
\epsilon' =\theta_*\sqrt{\frac{1}{\delta (q^2 \cdot \frac{(1+\alpha)\lambda_1}{\beta_1\lambda_2}-1)}}.
\end{equation*}

For this special case, we note that a multiplicative increase in $\beta_1$, by say $c_1>1,$ can be compensated by letting the attack happen while training over neural networks $f_{c_1\bm{1},\w}$. Similarly, a multiplicative increase in $\theta_*$, by say $c_2>1$, can be compensated by letting the attack
happen while training over neural networks $f_{c_2\bm{1},\w}$, since 
$$c_2\theta_*\sqrt{\frac{1}{\delta (c_2^2 \cdot \frac{(1+\alpha)\lambda_1}{\beta_1\lambda_2}-1)}}<\theta_*\sqrt{\frac{1}{\delta ( \frac{(1+\alpha)\lambda_1}{\beta_1\lambda_2}-1)}}. $$

If Algorithm \ref{tronalgo} 
is used for training over neural networks
$f_{\bm{1},\w}$ and
$f'_{q^2\bm{1},\w},~q>1,$
with the same $\{\A_i\}_{i=1}^k$ matrices,
then one can choose a common $\M$ for both the instances such that
while facing the same output-poisoning adversary,
the accuracy of recovering the true weights in class
$f'_{q^2\bm{1},\w}$
is higher than the accuracy of recovering the true weights in class
$f_{\bm{1},\w}.$
In other words, increasing the outer layer weights
via higher values of $q>1$
improves the accuracy-related defence of Algorithm ~\ref{tronalgo}
against the same type of adversarial attack. 
To this end,
we demonstrate via a simulation-based experiment
the accuracy advantage gained by
upscaling the outer layer weights
(see Appendix \ref{ssc:highqsim}).



\section{Simulation study}\label{sec:sim} 





In this section,
we conduct a simulation study by training $\relu$ neural networks
on different input data distributions and
for different hyperparameter settings of Algorithm \ref{tronalgo}.
The purpose of the simulation study is twofold,
to empirically validate
the relative theoretical performance bounds of Algorithm \ref{tronalgo}
and to compare Algorithm \ref{tronalgo} with SGD.
The code for our simulations can be found at
\url{https://github.com/papamarkou/neurotron_experiments}.

Eight setups are included in our simulation study.
For each setup, independent and identically distributed input data samples $\x_{t_i},~i=1,\ldots,b,$
are drawn from one of the following four distributions:
standard normal $\mathcal{N}(\mu=0, \sigma^2=1)$,
normal $\mathcal{N}(\mu=0, \sigma^2=9)$ with mean $\mu=0$ and variance $\sigma^2=9$,
$\mbox{Laplace}(\mu=0, b=2)$ with mean $\mu=0$ and scale $b=2$,
or Student's t-distribution $\mbox{t}(\nu=4)$ with $\nu=4$ degrees of freedom.
For each setup associated with an input data distribution,
either the noise bound $\theta_{*}$
or the probability $\beta$ of adversarial attack vary,
while the remaining hyperparameters are fixed.
To sum up, in each of the eights setups,
one out of four possible distributions is selected to sample input data,
and one of the two hyperparameters $\theta_{*}$ or $\beta$ vary.



To run a simulation for a given setup,
we initially sample a point $\w_* \sim {\cal N}(0,\mathbf{\I}_r)$ and
sample the sensing matrices
as explained in Remark~\ref{theorem1_remarks}-1.
We then train our $\relu$ network via Algorithm \ref{tronalgo} to approximate $\w_*$,
starting from the weight initialization $\w_1={\bf 1}\in\R^r$
at the first iteration.
At the $t$-th iteration of Algorithm \ref{tronalgo},
we draw iid input data samples $\x_{t_i},~i=1,\ldots,b,$
from a distribution fixed throughout the run,
selected among the four aforementioned possible distributions.
Given input data point $\x_{t_i}$,
we instantiate a data-poisoning attack
without explicitly checking for consistency with the assumptions of Theorem \ref{main2};
we sample $\alpha_{t_i}$
from $\mbox{Bernoulli}(\beta (\x_{t_i}))$. Thus the probability of attack is $\beta (\x_{t_i}) =\mbox{Pr} (\alpha_{\x_{t_i}}=1),$
and if $\alpha_{t_i}=1$, we set the additive distortion as,
$
\xi_{t_i} = \theta_*\ind{i\mymod{2} = 0} - \theta_*\ind{i \mymod{2} \neq 0},
$
where $\ind{}$ denotes the indicator function.
We run SGD similarly to Algorithm \ref{tronalgo}.

\begin{table}[t!]
\centering
\begin{tabular}{|l|r|r|r|r|r|r|r|r|} 
\hline
Data &
\multicolumn{2}{|c|}{$\mathcal{N}(0, 1)$} &
\multicolumn{2}{c|}{$\mbox{t}(4)$} &
\multicolumn{2}{c}{$\mathcal{N}(0, 9)$} &
\multicolumn{2}{|c|}{$\mbox{Laplace}(0, 2)$} \\ \hline
$\theta_*$ &
Varying & 0.25 &
Varying & 0.25 &
Varying & 0.25 &
Varying & 0.25 \\ \hline
$\beta$ &
0.5 & Varying &
0.5 & Varying &
0.5 & Varying &
0.5 & Varying \\ \hline
$\eta$ &
0.0001 & 0.0001 & 0.0001 & 0.0001 &
0.00005 & 0.00005 & 0.00005 & 0.00005 \\ \hline
n & 100 & 100 & 100 & 100 & 50 & 50 & 50 & 50 \\ \hline
r & 25 & 25 & 25 & 25 & 25 & 25 & 25 & 25 \\ \hline
b & 16 & 16 & 16 & 16 & 16 & 16 & 16 & 16 \\ \hline
k & 10 & 10 & 10 & 10 & 10 & 10 & 10 & 10 \\ \hline
\end{tabular}
\caption{
Configuration of hyperparameters (rows) across eight experimental setups (columns).
Each setup arises from a combination of an input data distribution
and of a varying hyperparmater.
The varying hyperparameter is either the noise bound
$\theta_*\in\{0, 0.125, 0.25, 0.5, 1, 2, 4\}$ or
the probability of adversarial attack
$\beta\in\{0.005, 0.05, 0.1, 0.2, 0.5, 0.9\}$.
The rest of hyperparameters are the following:
learning rate $\eta$,
input data dimension $n$,
filter size $r$,
batch size $b$, and
network width $k$.}
\label{table:hps}
\end{table}

Table \ref{table:hps} summarizes
the configuration of hyperparameters
across the eight simulation setups.
When the noise bound $\theta_*$ varies,
it takes values in $\{0, 0.125, 0.25, 0.5, 1, 2, 4\}$
and the probability $\beta$ of adversarial attack is fixed to $0.5$.
When $\beta$ varies,
it takes values in $\{0.005, 0.05, 0.1, 0.2, 0.5, 0.9\}$
and $\theta_*$ is fixed to $0.25$.
Based on empirical tuning,
the learning rate is set to $\eta = 0.0001$
when the input data distribution is
$\mathcal{N}(\mu=0, \sigma^2=1)$ or $\mbox{t}(\nu=4)$,
and to  $\eta = 0.00005$
when the input data distribution is
$\mathcal{N}(\mu=0, \sigma^2=9)$ or $\mbox{Laplace}(\mu=0, b=2)$.
The input data dimension is set to $n=100$
for data sampled from $\mathcal{N}(\mu=0, \sigma^2=1)$ or $\mbox{t}(\nu=4)$,
whereas it is set to $n=50$
for data sampled from $\mathcal{N}(\mu=0, \sigma^2=9)$ or $\mbox{Laplace}(\mu=0, b=2)$;
smaller dimension $n$ is used in the latter case
due to higher variance in the input data,
which can affect the numerical stability of Algorithm \ref{tronalgo} and of SGD.
The filter size, batch size and network width are set to
$r=25$, $b=16$ and $k=10$ across all setups.

\begin{figure}[t!]
\centering
\begin{subfigure}{.5\textwidth}
  \centering
  \includegraphics[width=1.\linewidth]{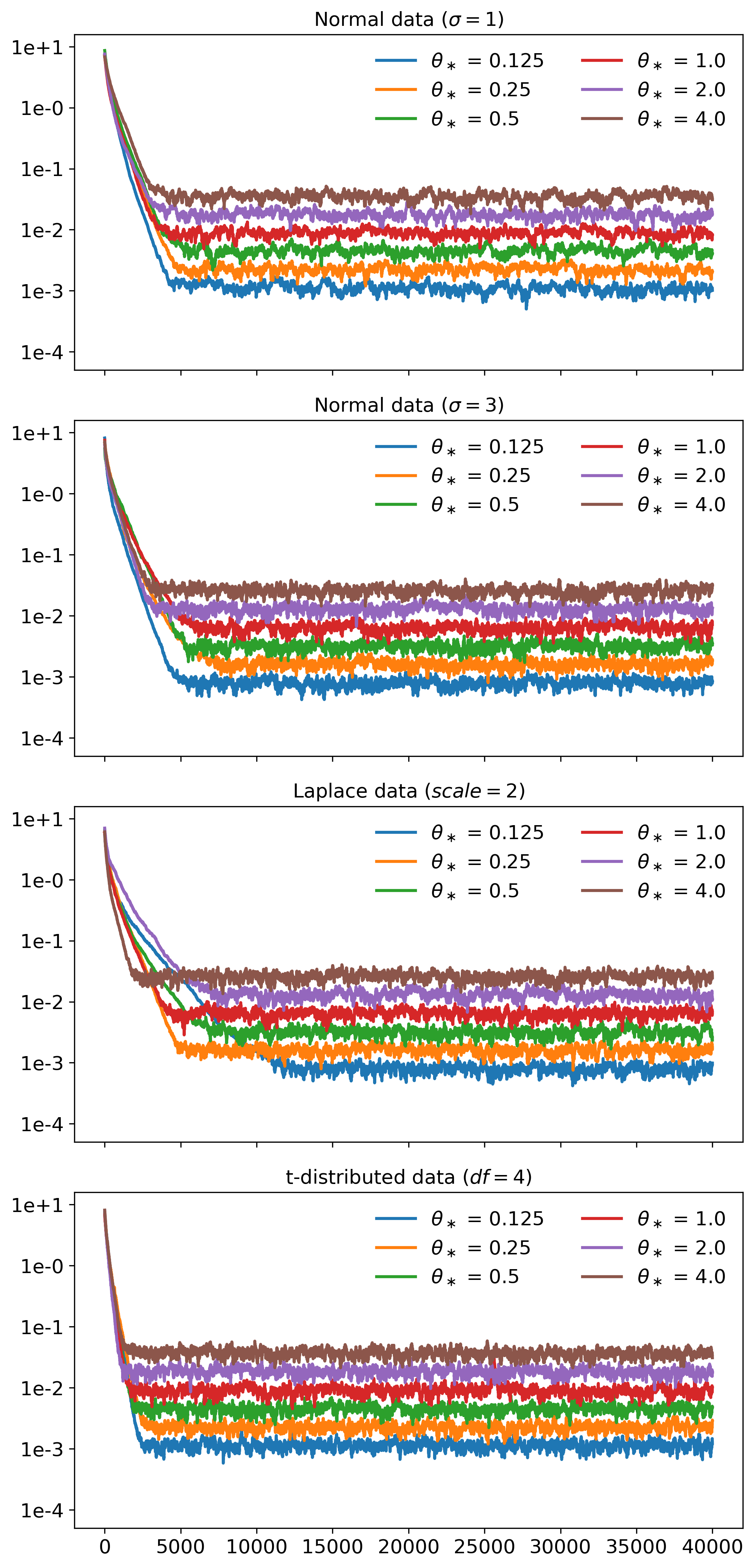}
  \caption{Different $\theta_{\star}$ values.}
  \label{fig:tron_varying_theta}
\end{subfigure}%
\begin{subfigure}{.5\textwidth}
  \centering
  \includegraphics[width=1.\linewidth]{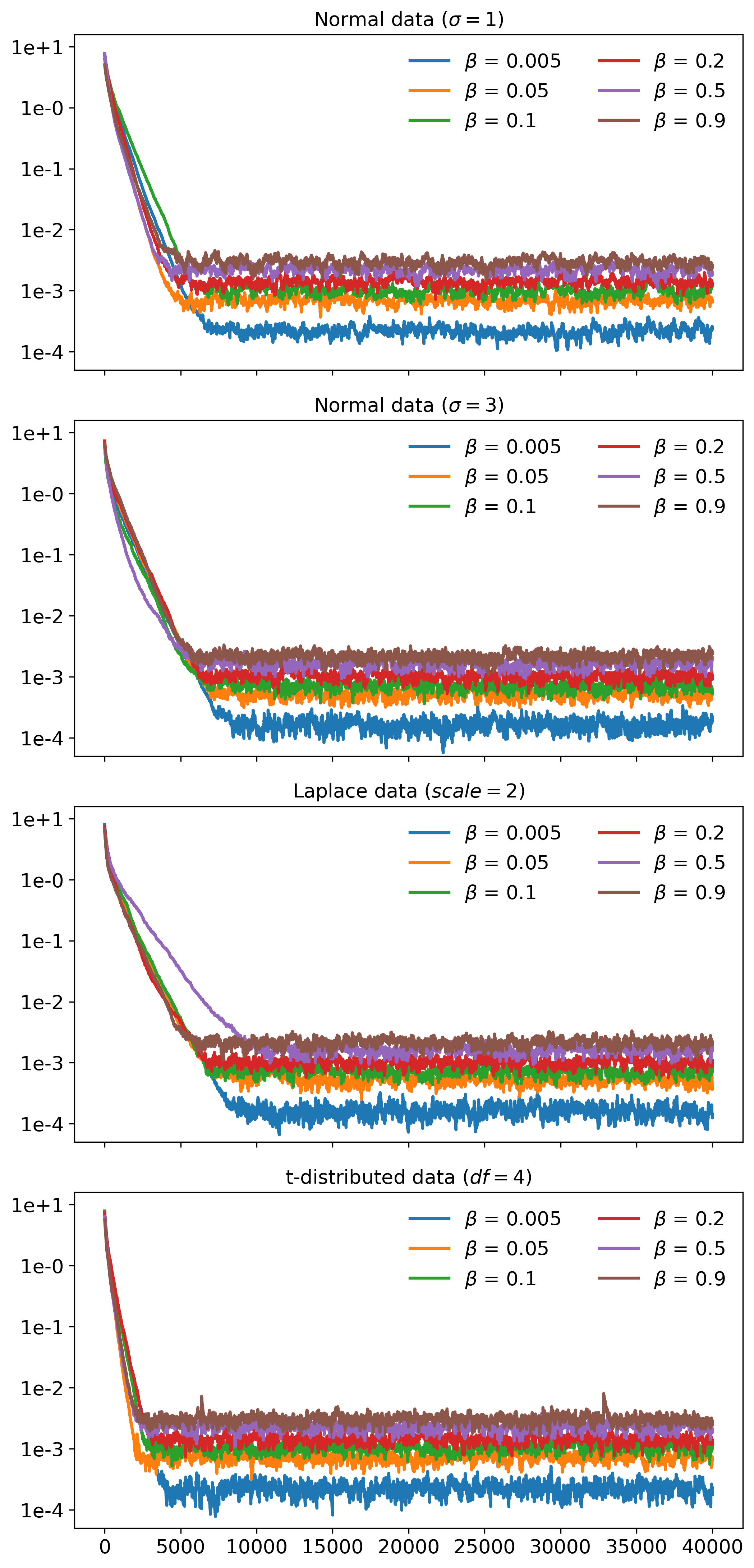}
  \caption{Different $\beta$ values.}
  \label{fig:tron_varying_beta}
\end{subfigure}
\caption{Simulation-based validation
of Theorem \ref{main2} regarding the performance
of Algorithm \ref{tronalgo} (Neuro-Tron).
(a): Neuro-Tron parameter recovery errors
per input data distribution
for different adversarial noise bounds $\theta_{\star}$.
(b): Neuro-Tron parameter recovery errors
per input data distribution
for different probabilities $\beta$ of adversarial attack.
}
\label{fig:tron_varying_theta_and_beta}
\end{figure}

Performance is measured in terms of the parameter recovery error
$\norm{\w_t - \w_{\star}}$
at iteration $t$ of Algorithm \ref{tronalgo} and of SGD,
where $\norm{\cdot}$ denotes the Euclidean norm.
In all figures of this section and of appendices
\ref{app:tron_vs_sgd}, \ref{app:tron_no_attack} and \ref{ssc:highqsim},
the vertical and horizontal axes display
recovery errors and iterations, respectively.
Recovery error tick mark labels are shown in $\log_{10}$ scale,
while the corresponding tick marks are shown in the original scale.

\begin{figure}[t!]
\centering
\begin{subfigure}{.5\textwidth}
  \centering
  \includegraphics[width=1.\linewidth]{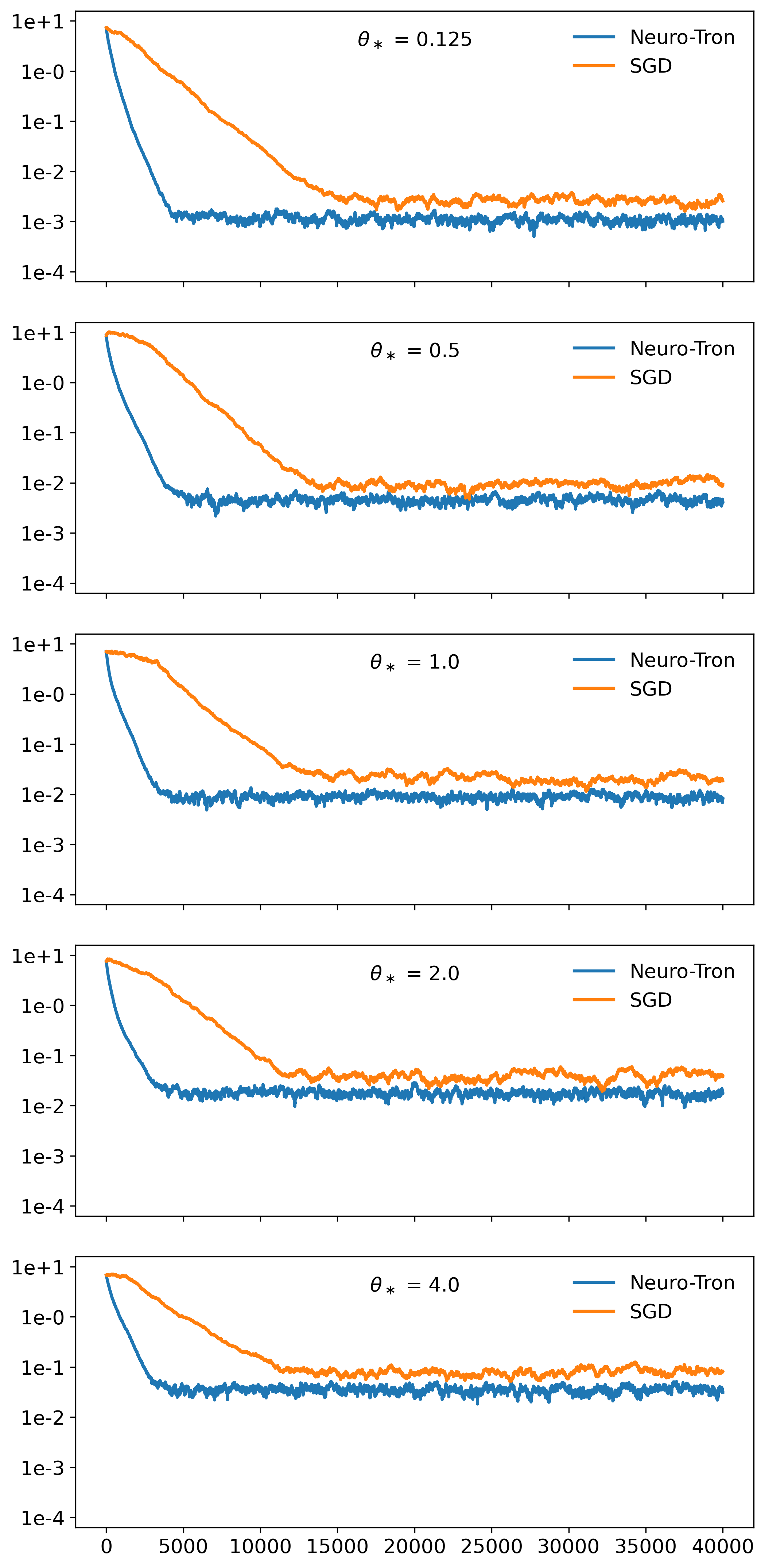}
  \caption{Neuro-Tron versus SGD per $\theta_{\star}$ value.}
  \label{fig:tron_vs_sgd_theta_normal_small_var}
\end{subfigure}%
\begin{subfigure}{.5\textwidth}
  \centering
  \includegraphics[width=1.\linewidth]{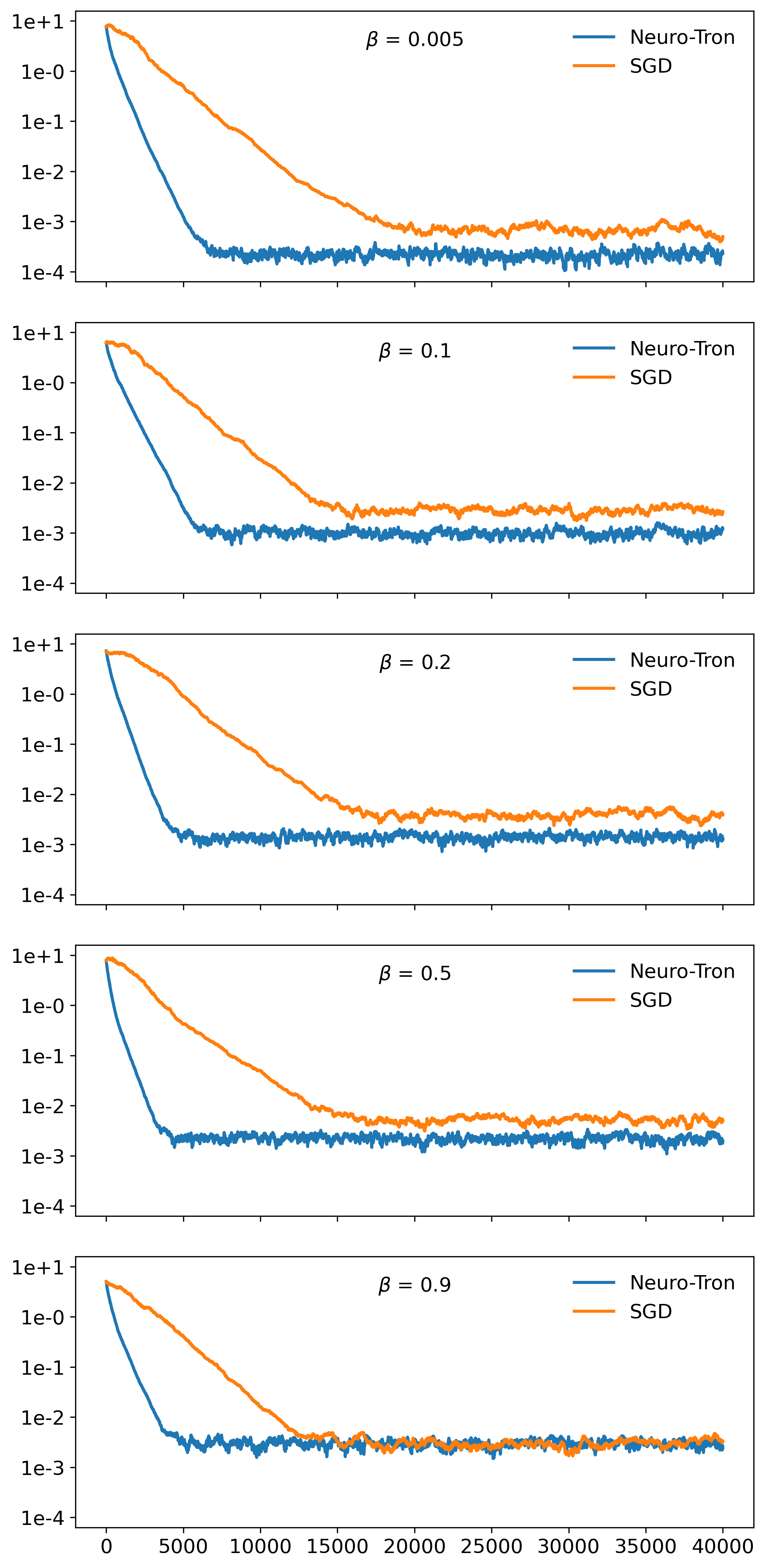}
  \caption{Neuro-Tron versus SGD per $\beta$ value.}
  \label{fig:tron_vs_sgd_beta_normal_small_var}
\end{subfigure}
\caption{Simulation-based comparison
between Algorithm \ref{tronalgo} (Neuro-Tron) and SGD.
Input data are sampled from $\mathcal{N}(\mu=0, \sigma^2=1)$.
(a): Parameter recovery errors
for different adversarial noise bounds $\theta_{\star}$.
(b): Parameter recovery errors
for different probabilities $\beta$ of adversarial attack.
}
\label{fig:tron_vs_sgd_normal_small_var}
\end{figure}

\begin{figure}[t!]
\centering
\begin{subfigure}{.5\textwidth}
  \centering
  \includegraphics[width=1.\linewidth]{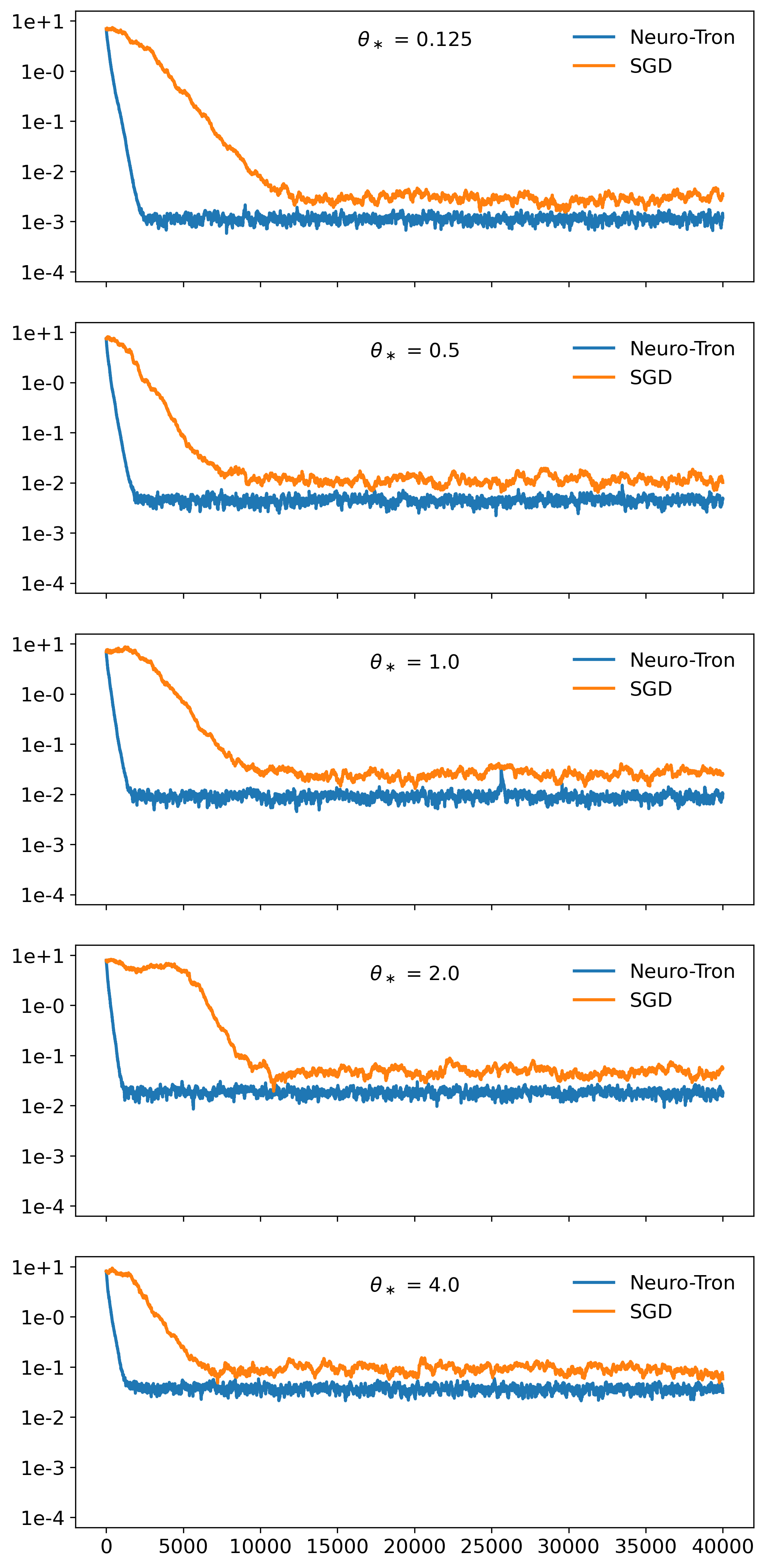}
  \caption{Neuro-Tron vs SGD per $\theta_{\star}$ value.}
  \label{fig:tron_vs_sgd_theta_tdist}
\end{subfigure}%
\begin{subfigure}{.5\textwidth}
  \centering
  \includegraphics[width=1.\linewidth]{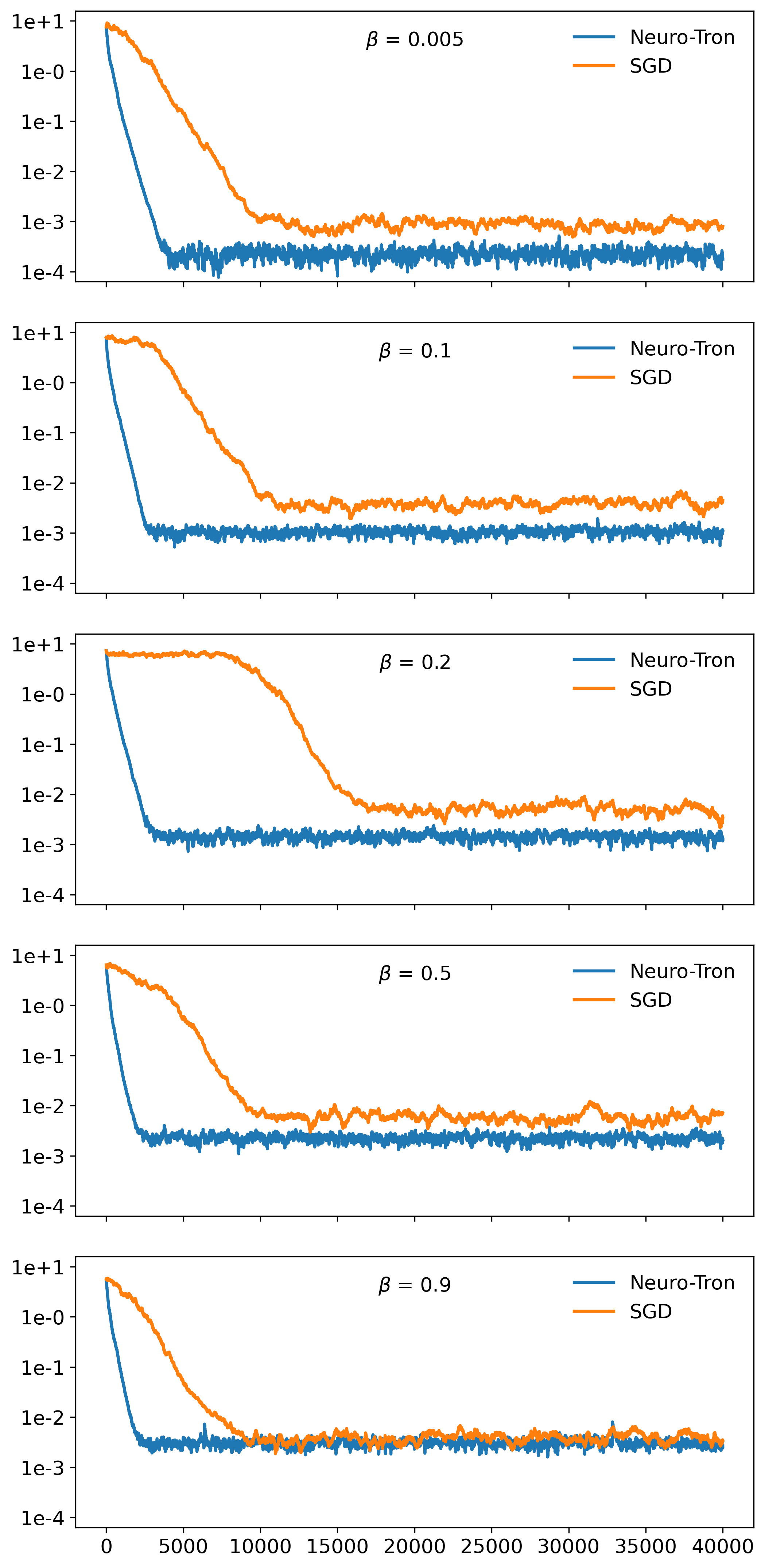}
  \caption{Neuro-Tron vs SGD per $\beta$ value.}
  \label{fig:tron_vs_sgd_beta_tdist}
\end{subfigure}
\caption{Simulation-based comparison
between Algorithm \ref{tronalgo} (Neuro-Tron) and SGD.
Input data are sampled from Student's $\mbox{t}(\nu=4)$.
(a): Parameter recovery errors
for different adversarial noise bounds $\theta_{\star}$.
(b): Parameter recovery errors
for different probabilities $\beta$ of adversarial attack.
}
\label{fig:tron_vs_sgd_tdist}
\end{figure}

Figure~\ref{fig:tron_varying_theta_and_beta} provides a
simulation-based validation
of Theorem \ref{main2} regarding the performance
of Algorithm \ref{tronalgo} (Neuro-Tron).
In each plot of Figure~\ref{fig:tron_varying_theta_and_beta},
input data are sampled from a fixed distribution.
On the left-hand side of Figure~\ref{fig:tron_varying_theta_and_beta},
increasing the magnitude of attack (noise bound) $\theta_{\star}$
increases the parameter recovery error.
On the right-hand side of Figure~\ref{fig:tron_varying_theta_and_beta},
increasing the probability of attack $\beta$
increases the parameter recovery error.

To further validate Neuro-Tron via simulation,
Figure~\ref{fig:no_attack} in Appendix~\ref{app:tron_no_attack} provides
parameter recovery errors in the absence of data-poisoning attack
($\theta_{\star}=0$).
Recovery errors are in the vicinity of $10^{-14}$ for $\theta_{\star}=0$
across different input data distributions,
demonstrating the capacity of Neuro-Tron to recover network parameters
under no attack.
Moreover, Figure~\ref{fig:no_attack}
shows an anticipated degradation in parameter recovery
under relatively small magnitude of attack ($\theta_{\star}=0.125$)
when comparing to no attack ($\theta_{\star}=0$).

Figures
\ref{fig:tron_vs_sgd_normal_small_var},
\ref{fig:tron_vs_sgd_tdist}
in this section
and Figures
\ref{fig:tron_vs_sgd_laplace},
\ref{fig:tron_vs_sgd_normal_large_var}
in Appendix~\ref{app:tron_vs_sgd}
provide a simulation-based comparison between
Neuro-Tron and SGD
for different input data distributions,
noise bounds $\theta_{\star}$
and probabilities $\beta$ of attack.
These figures provide empirical evidence
that Neuro-Tron attains
smaller parameter recovery error
and faster rate of convergence
than SGD under data-poisoning attacks.

We note that even with input data distributions,
such as $\mbox{Laplace}(\mu=0, b=2)$ which have tails heavier than the Gaussian, we see in 
Figure~\ref{fig:tron_vs_sgd_laplace} that
Neuro-Tron retains its advantage over SGD.
More strikingly, Neuro-Tron outperforms SGD
under Student's $\mbox{t}(\nu=4)$ distribution as seen in Figure \ref{fig:tron_vs_sgd_tdist}.
Note that $\mbox{t}(\nu=4)$ has infinite kurtosis (fourth moment),
and therefore it is not covered by the assumptions of Theorem~\ref{main2};
nevertheless, our simulations demonstrate that Neuro-Tron attains
analogous parameter recovery accuracy with $\mbox{t}(\nu=4)$ as it does
with the other three input data distributions.

\clearpage 
\section{Conclusion}\label{sec:conc}

In this paper, we provide the first
provably robust training algorithm
for a class of finite-width neural networks
under a data-poisoning attack.
In particular,
we have constructed
an iterative stochastic gradient-free algorithm which,
up to a given level of parameter approximation accuracy
and level of probabilistic confidence,
performs supervised learning on
a finite-width  neural network
in the presence of a malicious oracle
adding noise to some true continuous output.
We also establish that our performance guarantees
are nearly-optimal in the worst case of
attack on every output point.

Three open questions arise based on the present results.
Firstly, it remains to extend our results
to broader classes of neural networks and to data distributions with lesser number of moments being finite than assumed in Theorem~\ref{main2}.
Secondly, an open question is to characterize
the approximation accuracy and confidence trade-off
of Theorem \ref{main2}
as a function of the probability
of adversarial attack.
Thirdly, alternative adversarial attacks
can be considered,
conducting non-additive distortions to the output data
or corrupting the input data. 

\section*{Acknowledgements}
We would like to thank Amitabh Basu for extensive discussions on various parts of this paper. The first author's research is supported by NSF DMS 2124222. The second author would like to thank the MINDS Data Science Fellowship of
Johns Hopkins University
for supporting this work. The second author would also like to acknowledge the extensive discussions on this topic with Anup Rao, Sridhar Mahadevan, Pan Xu and Wenlong Mou when he was interning at Adobe, San Jose, during the summer of $2019$.


\bibliographystyle{elsarticle-num-names} 
\bibliography{references}

\clearpage 

\appendix

\section{Proof of Theorem \ref{main2}}\label{sec:proofthm1}



\begin{proof}
Between consecutive iterates of the algorithm we have,
\begin{align*}
\nonumber \norm{\w^{(t+1)}-\w^*}^2 
&= \norm{\w^{(t)} +\eta\g^{(t)} - \w^*}^2 \\
&= \norm{\w^{(t)}-\w^*}^2 + \eta^2\norm{\g^{(t)}}^2 + 2\eta\langle\w^{(t)}-\w^*,\g^{(t)}\rangle.  
\end{align*}
Let the training data sampled till the iterate $t$ be $S_t := \bigcup_{i=1}^t s_i$. We overload the notation to also denote by $S_t$, the sigma-algebra generated by the samples seen {\it and the $\alpha$s} till the $t$-th iteration. Conditioned on $S_{t-1}$ , $\w_t$ is determined and $g_t$ is random and dependent on the choice of $\s_t$ and $\{\alpha_{t_i},\xi_{t_i} \mid i =1,\ldots,b\}$. We shall denote the collection of random variables $\{\alpha_{t_i} \mid i =1,\ldots,b\}$ as $\alpha_t$. Then taking conditional expectations w.r.t $S_{t-1}$ of both sides of the above equation we have,
\begin{align}\label{DBothTerms:b}
 \nonumber & \E_{s_t,\alpha_t} \Bigg [ \norm{\w^{(t+1)} - \w^*}^2 \bigg| S_{t-1} \Bigg ] \\
 \nonumber &= 
 \underbrace{2\frac{\eta}{b} \cdot  \sum_{i=1}^b \E_{\x_{t_i},\alpha_{t_i}} \Bigg [ \Big \langle \w^{(t)} - \w^* , \M\Big (y_{t_i} - f_{\w^{(t)}}(\x_{t_i}) \Big) \x_{t_i}   \Big \rangle \bigg| S_{t-1} \Bigg ]}_{\text{Term }1}\\
&+ \underbrace{\eta^2 \E_{\x_{t_i},\alpha_{t_i}} \Big [ \norm{\g^{(t)}}^2 \bigg| S_{t-1} \Bigg ]}_{\text{Term } 2}  +  \E_{s_t,\alpha_t} \Bigg [ \norm{\w^{(t)} - \w^*}^2 \bigg| S_{t-1} \Bigg ].
\end{align}


We provide the bound for Term 1 in the Appendix \ref{ssc:A1} and arrive at
\begin{eqnarray}\label{term1}&&\hspace{-0.24 in}\text{Term } 1 \\ &\leq& -\eta(1+\alpha)\cdot \lambda_1 \cdot \norm{\w^{(t)} - \w^*}^2+ 2\eta \theta \lambda_2 \cdot \mathbb{E}\Big [\beta(\x_{t_1})\norm{\x_{t_1}} \bigg|S_{t-1} \Big ]  \cdot  \norm{\w^{(t)}-\w^*}. \nonumber
\end{eqnarray}

Now we split the {\rm Term 2} in the RHS of equation \ref{DBothTerms:b} as follows:
\begin{align}\label{term2}
\nonumber&\mathbb{E} \Bigg [\norm{\eta \g^{(t)}}^2 \bigg| S_{t-1}\Bigg ]\\ 
&=\frac{\eta^2}{b^2}\bigg(\mathbb{E}\Bigg[\sum_{i=1}^b(y_{t_i} - f_{\w^{(t)}}(\x_{t_i}))^2\cdot \norm{\M \x_{t_i} }^2\bigg|S_{t-1}\Bigg] \nonumber \\
&\quad + \mathbb{E}\Bigg[\sum_{i=1}^b\sum_{j=1,j \neq i}^{b}(y_{t_i} - f_{\w^{(t)}}(\x_{t_i}))(y_{t_j} - f_{\w^{(t)}}(\x_{t_j}))\cdot \x_{t_j}^{\top}\M^{\top}\M \x_{t_i}  \bigg|S_{t-1}\Bigg]\bigg)\nonumber\\
&=: \text{Term }21+\text{Term }22.
\end{align}


We separately upperbound the {\rm Term 21} and {\rm Term 22} as outlined in the Appendix \ref{ssc:A2} and \ref{ssc:A3} respectively and arrive at, 
\begin{eqnarray}
\label{term21} \text{Term } 21&\leq \frac{\eta^2\lambda_2^2}{b}\left(c^2{\rm m}_4\norm{\w^{(t)}-\w^*}^2+ 2c\theta\beta_3 \norm{\w^{(t)}-\w^*}+\theta^2\beta_2\right).
\end{eqnarray}
\begin{eqnarray}
&&\hspace{-0.3 in}\text{Term } 22\label{term22} \\ \nonumber  &\leq&\frac{\eta^2(b^2-b)}{b^2}\left[\theta^2\lambda_2^2\beta_1^2+2\theta\lambda_2^2\beta_1c{\rm m}_2\norm{\w^{(t)}-\w^*}+\lambda_2^2c^2{\rm m}_2^2\norm{\w^{(t)}-\w^*}^2  \right].
\end{eqnarray}

Next we take total expectations of both sides of equations \ref{term1}, \ref{term21} and \ref{term22} recalling that the conditional expectation of functions of $\x_{t_i}$ w.r.t. $S_{t-1}$ are random variables which are independent of the powers of $\norm{\w^{(t)} - \w^*}$. Then we substitute the resulting expressions into the RHS of equation \ref{DBothTerms:b} to get,
\begin{align}\label{eq:total}
\nonumber \mathbb{E}& \bigg [ \norm{\w^{(t+1)}-\w^*}^2 \bigg] \\ \nonumber
&\leq \bigg[ 1 + \eta^2\lambda^2_2c^2 \left( {\rm m}_2^2\left(1-\frac{1}{b}\right) +\frac{{\rm m}_4}{b} \right)  -\eta\lambda_1(1+\alpha)\bigg]\cdot \mathbb{E} \bigg [ \norm{\w^{(t)}-\w^*}^2 \bigg]  \\
&+ \bigg[2\eta^2\lambda^2_2c\theta\left( \beta_1{\rm m}_2\left(1-\frac{1}{b}\right) +\frac{\beta_3}{b} \right)  + 2\eta\lambda_2 \cdot \beta_1 \theta \bigg]\cdot \mathbb{E} \bigg [ \norm{\w^{(t)}-\w^*} \bigg]\nonumber \\
&+ \eta^2\theta^2\lambda^2_2\left( \beta_1^2\left(1-\frac{1}{b}\right) +\frac{\beta_2}{b} \right).
\end{align}

\textbf{Case I : Realizable, $\theta = 0$.}

Here the recursion above simplifies to,
\begin{align}\label{eq:totalnonoise}
\mathbb{E}&\bigg [ \norm{\w^{(t+1)}-\w^*}^2 \bigg] \\
\nonumber &\leq \bigg[ 1 + \eta^2\lambda^2_2c^2 \left( {\rm m}_2^2\left(1-\frac{1}{b}\right) +\frac{m_4}{b} \right)- \eta\lambda_1(1+\alpha)\bigg]\cdot \mathbb{E} \bigg[ \norm{\w_{t}-\w^*}^2 \bigg] 
\end{align}

Let $\kappa = 1 + \eta^2\lambda^2_2c^2 \left( {\rm m}_2^2(1-1/b) +{\rm m}_4/b \right)- \eta\lambda_1(1+\alpha) $. Thus, for all $t \in \mathbb{Z}^+$,
\begin{align*}
\mathbb{E} \bigg [ \norm{\w^{(t)}-\w^*}^2 \bigg] 
\leq \kappa^{t-1} \mathbb{E} \bigg [ \norm{\w^{(1)}-\w^*}^2 \bigg]. 
\end{align*}
Recalling that $c^2 = (1+\alpha)\lambda_3$, we can verify that the choice of step size given the Theorem is $\eta = \frac{1}{\gamma}\cdot \frac{\lambda_1(1+\alpha)}{\lambda_2^2c^2(\mathrm{m}_4/b+\mathrm{m}_2^2(1-1/b))}$ and the assumption on $\gamma$ ensures that for this $\eta$, $\kappa = 1 - \frac{\gamma - 1}{\gamma^2}\frac{\lambda_1^2}{\lambda_2^2\lambda_3(\mathrm{m}_4/b+\mathrm{m}_2^2(1-1/b))} \in (0,1)$.
Therefore, for ${\rm T} = {\large \mathcal{O}}\Big( \log \left( \frac{\norm{\w^{(1)}-\w^*}} {\epsilon^2\delta}\right)\Big)$, we have 
\[ \mathbb{E}\Big[\norm{\w^{({\rm T})}-\w^*}^2 \Big] \leq \epsilon^2\delta . \]    
The conclusion now follows from Markov's inequality. 

\textbf{Case II : Realizable + Adversarial Noise, $\theta \in (0,\theta_*)$.}

Note that the linear term $\norm{\w^{(t)}-\w^*}$ in equation \ref{eq:total} is a unique complication that is introduced here because of the absence of distributional assumptions on the noise in the labels. We can now upperbound the linear term using the {\rm AM-GM} inequality as follows - which also helps decouple the adversarial noise terms from the distance to the optima. Then equation \ref{eq:total} implies, 
\begin{align*}
\Delta_{t+1} 
&\leq \Bigg[\eta^2\lambda^2_2c^2\left( (\beta_1{\rm m}_2+{\rm m}_2^2)\left(1-\frac{1}{b}\right)
+\frac{\beta_3+{\rm m}_4}{b} \right)
-\eta\lambda_2\left(\frac{\lambda_1(1+\alpha)}{\lambda_2}
-\beta_1\right)+1\Bigg]
\Delta_{t}\\
&+\theta^2 \Big(\eta^2\lambda^2_2\left( (\beta_1^2+\beta_1{\rm m}_2)\left(1-\frac{1}{b}\right)
+\frac{\beta_2+\beta_3}{b} \right)
+\eta\lambda_2 \beta_1\Big),
\end{align*}
where $\Delta_t := \mathbb{E} \Big [\norm{\w^{(t)}-\w^*}^2 \Big ]$.
We introduce the following notation:
$\eta' := \eta\lambda_2$,
$b_* := \frac{\lambda_1(1+\alpha)}{\lambda_2}-{\beta}_1$,
$c_1=c^2\left( (\beta_1{\rm m}_2+{\rm m}_2^2)\left(1-\frac{1}{b}\right) +\frac{\beta_3+{\rm m}_4}{b} \right)$, 
$c_2 := \theta^2 \Big(\eta^2\lambda^2_2\left( (\beta_1^2+\beta_1{\rm m}_2)\left(1-\frac{1}{b}\right) +\frac{\beta_2+\beta_3}{b} \right)$ and $c_3 = \theta^2 \beta_1$. Then the dynamics of the algorithm is given by,
\begin{align*}
\Delta_{t+1} \leq (1-\eta'b_* + \eta'^2c_1)\Delta_t + \eta'^2c_2 + \eta'c_3. 
\end{align*}

We note that the above is of the same form as lemma \ref{recurse} in the Appendix with $\Delta_1 = \norm{\w_1 - \w^*}^2$. We invoke the lemma with $\epsilon'^2 := \epsilon^2 \delta$ such that equation \ref{condmain1} holds. 

This along with the bound on noise that $\theta \in (0, \theta_*)$ ensures that,  $\frac{c_3}{b_*} = \frac{\theta^2}{{\rm c}_{\rm trade-off}} = \frac{\theta^2}{\theta_*^2}\cdot \epsilon^2\delta < \epsilon^2 \delta < \Delta_1$ as required by Lemma \ref{recurse}, Appendix \ref{app:rec}. 


Recalling the definition of ${\rm c}_{\rm rate}$ as given in the theorem statement we can see that $\frac{\frac{c_2}{c_1}+\gamma \cdot \frac{c_3}{b_*}}{\gamma - 1} = \frac{\theta^2}{{\rm c}_{\rm rate}}$ and hence we can read off from Lemma \ref{recurse}, Appendix \ref{app:rec}, that at the value of ${\rm T}$ as specified in the theorem statement we have, 


\[ \Delta_{\rm T} = \mathbb{E} \Big [\norm{\w_{\rm T}-\w^*}^2 \Big ] \leq \epsilon^2 \delta \]
and the needed high probability guarantee follows by Markov inequality.
\end{proof}

\section{Bounds needed in the proof in Section \ref{sec:proofthm1}}\label{sec:proofthm1_2}

In the following sub-sections we provide the upperbounds for {\rm Term 1}, {\rm Term 21} and {\rm Term 22} in the previous appendix. 

\subsection{Upperbound for {\rm Term 1} }\label{ssc:A1}
For Term 1 in equation (\ref{DBothTerms:b}) we proceed by observing that conditioned on $S_{t-1}$, $\w^{(t)}$ is determined while $\w^{(t+1)}$ and $\g^{(t)}$ are random. Thus we compute the following conditional expectation (suppressing the subscripts of $\x_{t_i}, \alpha_{t_i}$),

\begin{align}
\nonumber \text{Term }1 & = \mathbb{E} \Bigg [ 2\eta \langle\w^{(t)}-\w^*,\g^{(t)}\rangle \bigg| S_{t-1} \Bigg]  \\
\nonumber & = 2\frac{\eta}{b} \sum_{i=1}^b\mathbb{E}\Bigg[ \Big (f_{\w^*}(\x_{t_i}) + \alpha_{t_i}\xi_{t_i} - f_{\w^{(t)}}(\x_{t_i}) \Big) (\w^{(t)}-\w^*)^\top \M  \x_{t_i} \bigg| S_{t-1}\Bigg] \\
\nonumber & = 2\frac{\eta}{b}\sum_{i=1}^b\mathbb{E}\Bigg[\Big (f_{\w^*}(\x_{t_i}) - f_{\w^{(t)}}(\x_{t_i}) \Big)\cdot (\w^{(t)}-\w^*)^\top \M  \x_{t_i} \bigg| S_{t-1}\Bigg]  \nonumber \\
& \quad\quad + 2\frac{\eta}{b}\sum_{i=1}^b\mathbb{E}\Bigg[\alpha_{t_i}\xi_{t_i}(\w^{(t)}-\w^*)^\top \M \x_{t_i} \bigg| S_{t-1}\Bigg]\\
\nonumber & \leq \frac{-2\eta}{bk}\sum_{i=1}^b\sum_{j=1}^k\mathbb{E}\Bigg[\Big 
(\sigma({\w^{(t)}}^\top \A_j\x_{t_i}) - \sigma(\w^{*\top} \A_j\x_{t_i})  \Big)  (\w^{(t)}-\w^*)^\top \M  \x_{t_i} \bigg|S_{t-1}\Bigg] \\
&\qquad \qquad+ 2\frac{\eta \theta}{b}\sum_{i=1}^b\mathbb{E}\Bigg[\beta(\x_{t_i})\cdot |(\w^{(t)}-\w^*)^\top \M \x_{t_i}| \bigg|S_{t-1}\Bigg].
\end{align}


\noindent We simplify the first term above by recalling an identity proven in \cite{goel2018learning}, which we have reproduced here as Lemma \ref{lem:convo} in Appendix \ref{app:lem}. Thus we get, 
\begin{align*}
\nonumber &\mathbb{E} \Bigg [ 2\eta \langle \w^{(t)}-\w^*,\g^{(t)} \rangle \bigg| S_{t-1} \Bigg ]\\
\nonumber &\leq \frac{-\eta(1+\alpha)}{bk}\sum_{i=1}^b \sum_{j=1}^k \mathbb{E} \Bigg[ 
(\w^{(t)} - \w^* )^\top \A_j\x_{t_i} (\w^{(t)}-\w^*)^\top \M  \x_{t_i} \bigg|S_{t-1} \Bigg] \\ \nonumber 
& \quad\quad + 2\frac{\eta \theta_*}{b} \sum_{i=1}^b\norm{\w^{(t)}-\w^*}\cdot \mathbb{E}\Bigg[\beta(\x_{t_i})\norm{\M \x_{t_i}} \bigg|S_{t-1}\Bigg]\\
\nonumber &\leq-\eta(1+\alpha) (\w^{(t)} - \w^*)^\top \;\bar{\A} \; \mathbb{E}\Big [\x_{t_1} \x_{t_1}^{\top} \bigg|S_{t-1} \Big ]\; \M^\top\; (\w^{(t)}-\w^*) \\ \nonumber 
& \quad \quad + 2\eta \theta_* \norm{\w^{(t)}-\w^*} \sqrt{\lambda_{\max}(\M^\top\M)} \cdot \mathbb{E}\Big [\beta(\x_{t_1})\norm{\x_{t_1}} \bigg|S_{t-1} \Big ]\\
\nonumber &\leq -\eta(1+\alpha) \cdot \lambda_{\min}\Bigg (\bar{\A} \mathbb{E}\Big [\x_{t_1} \x_{t_1}^{\top} \bigg|S_{t-1} \Big ] \M^\top \Bigg) \cdot \norm{\w^{(t)} - \w^*}^2
\\ \nonumber 
& \quad \quad+2\eta \theta_* \cdot \mathbb{E}\Big [\beta(\x_{t_1})\norm{\x_{t_1}} \bigg|S_{t-1} \Big ]   \cdot \sqrt{\lambda_{\max}(\M^T\M)}  \norm{\w^{(t)}-\w^*}\\
\nonumber&\leq -\eta(1+\alpha)\cdot \lambda_1 \cdot \norm{\w^{(t)} - \w^*}^2  \\ \nonumber 
& \quad \quad+ 2\eta \theta \lambda_2 \cdot \mathbb{E}\Big [\beta(\x_{t_1})\norm{\x_{t_1}} \bigg|S_{t-1} \Big ]  \cdot  \norm{\w^{(t)}-\w^*}.
\end{align*} 

\noindent We have invoked the i.i.d. nature of the data samples to invoke the definition of the $\lambda_1$ in above. 

\subsection{Upperbound for {\rm Term 21}}\label{ssc:A2}
For Term 21 in  equation (\ref{term2}) we get, 

\begin{align}
\nonumber &\text{Term }21\leq \frac{\eta^2\lambda^2_2}{b}  \cdot \mathbb{E}\Bigg[\big(f_{\w^*}(\x_{t_1}) + \alpha_{t_1}\xi_{t_1} - f_{\w^{(t)}}(\x_{t_1})\big)^2\cdot \norm{\x_{t_1} }^2 \bigg|S_{t-1}\Bigg] \\
\nonumber &\leq \frac{\eta^2\lambda^2_2}{b}  \cdot
\mathbb{E}\Bigg[\bigg ( \left(f_{\w^*}(\x_{t_1}) - f_{\w^{(t)}}(\x_{t_1})\right)^2 + 2\alpha_{t_1}\xi_{t_1}\left(f_{\w^*}(\x_{t_1})-f_{\w^{(t)}}(\x_{t_1})\right) + \alpha_{t_1}^2\xi_{t_1}^2\bigg )  
\cdot \norm{\x_{t_1}}^2 \bigg|S_{t-1}\Bigg]\\
& \leq \frac{\eta^2\lambda_2^2c^2}{b}\mathbb{E}\Bigg[
   \norm{\x_{t_1}}^4 \bigg| S_{t-1}\Bigg]   \norm{\w^{(t)}-\w^*}^2 
   + \frac{2\eta^2\lambda^2_2c\theta}{b}  \mathbb{E}\Bigg[ \beta(\x_{t_1})\norm{\x_{t_1} }^3 \bigg| S_{t-1}\Bigg]  \norm{\w^{(t)}-\w^*} \nonumber \\
  &\quad \quad + \frac{\eta^2\lambda^2_2 \theta^2}{b} \mathbb{E}\Bigg[\beta(\x_{t_1})\norm{\x_{t_1}}^2 \bigg| S_{t-1} \Bigg]\nonumber\\
&= \frac{\eta^2\lambda_2^2}{b}\left(c^2{\rm m}_4\norm{\w^{(t)}-\w^*}^2+ 2c\theta\beta_3 \norm{\w^{(t)}-\w^*}+\theta^2\beta_2\right). \nonumber
\end{align}

In the above lines we have invoked Lemma \ref{lem:diff_f_sq} from Appendix \ref{app:lem} twice to upperbound the term, $\vert \big(f_{\w^*}(\x^{(t)}) - f_{\w^{(t)}}(\x^{(t)})\big) \vert$ and we have defined, $$c^2 := (1+\alpha)^2\lambda_3 = \frac{(1+\alpha)^2}{k}\Big ( \sum_{i=1}^k \lambda_{\max}( {\A_i\A_i^\top} ) \Big ).$$

Next we proceed with Term 22 keeping in mind the independence of $x_{t_i}$ and $x_{t_j}$ for $i \neq j$,

\subsection{Upperbound for {\rm Term 22}}\label{ssc:A3}

For Term 22 in equation (\ref{term2}) we get,
\begin{align}
\nonumber &\text{Term }22\\
\nonumber &=\frac{\eta^2(b^2-b)}{b^2}\mathbb{E}\Bigg[(\alpha_{t_1}\xi_{t_1}+ f_{\w^{*}}(\x_{t_1})- f_{\w^{(t)}}(\x_{t_1}))(\alpha_{t_2}\xi_{t_2}+ f_{\w^{*}}(\x_{t_2}) - f_{\w^{(t)}}(\x_{t_2}))\cdot \x_{t_2}^{\top}\M^{\top}\M \x_{t_1}  \bigg|S_{t-1}\Bigg]\\
&\leq \frac{\eta^2(b^2-b)}{b^2}\Bigg[\theta^2 \left(\mathbb{E}_{x_{t_1}} \Bigg[ \beta(x_{t_1})\|\M x_{t_1}\|\bigg|S_{t-1}\Bigg]\right)^2\nonumber \\
&\quad\quad\quad + 2\theta\mathbb{E}_{x_{t_1}} \Bigg[ (f_{\w^{*}}(\x_{t_1}) - f_{\w^{(t)}}(\x_{t_1}))\|\M x_{t_1}\|\bigg|S_{t-1}\Bigg]\mathbb{E}_{x_{t_1}} \Bigg[\beta(x_{t_1})\|\M x_{t_1}\|\bigg|S_{t-1}\Bigg]  \nonumber \\
\nonumber &\quad\quad\quad\quad\quad +\mathbb{E}_{x_{t_1}} \Bigg[ (f_{\w^{*}}(\x_{t_1}) - f_{\w^{(t)}}(\x_{t_1}))\|\M x_{t_1}\|\bigg|S_{t-1}\Bigg]^2 \Bigg]\\
&\leq \frac{\eta^2(b^2-b)}{b^2}\left[\theta^2\lambda_2^2\beta_1^2+2\theta_*\lambda_2^2\beta_1c{\rm m}_2\norm{\w^{(t)}-\w^*}+\lambda_2^2c^2{\rm m}_2^2\norm{\w^{(t)}-\w^*}^2  \right]. \nonumber 
\end{align}

\section{Lemmas For Theorem \ref{main2}}\label{app:lem}

\begin{lemma}[Lemma 1, \cite{goel2018learning}]\label{lem:convo}
If ${\cal D}$ is parity symmetric distribution on $\R^n$ and $\sigma$ is an $\alpha-$Leaky $\relu$ then $\forall ~\a, {\ve b} \in \R^n$, 

\[\E_{\x \sim {\cal D}} \Big [ \sigma (\a^\top \x) {\ve b}^\top  \x  \Big ] = \frac{1+\alpha}{2}\E_{\x \sim {\cal D}} \Big [ (\a^\top \x) ({\ve b}^\top \x)  \Big ] \] .
\end{lemma}

\begin{lemma}\label{lem:diff_f_sq}
\begin{align*}
&(f_{\w_*}(\x) - f_{\w}(\x) )^2 \leq (1+\alpha)^2  \Big ( \frac{1}{k}\sum_{i=1}^k \lambda_{\max}( {\A_i\A_i^\top} ) \Big )\norm{\w_* - \w}^2   \norm{\x}^2.
\end{align*}
\end{lemma}

\begin{proof}

\begin{align*}
&\Big (f_{\w_*}(\x) - f_{\w}(\x) \Big )^2 \leq  \Big (\frac{1}{k} \sum_{i=1}^k \sigma \Big  ( \Big  \langle \A_i^\top \w_*  ,\x \Big  \rangle \Big  ) - \frac{1}{k} \sum_{i=1}^k \sigma \Big  ( \Big  \langle \A_i^\top \w ,\x \Big  \rangle \Big  ) \Big )^2\\
&\leq  \frac{1}{k} \sum_{i=1}^k \Big ( \sigma \Big  ( \Big  \langle \A_i^\top \w_*  ,\x \Big  \rangle \Big  ) - \sigma \Big  ( \Big  \langle \A_i^\top \w ,\x \Big  \rangle \Big  ) \Big )^2\\
&\leq  \frac{(1+\alpha)^2}{k} \sum_{i=1}^k \Big  \langle \A_i^\top \w_* - \A_i^\top \w ,\x \Big  \rangle ^2 = \frac{(1+\alpha)^2}{k} \sum_{i=1}^k \Big  ((\w_* - \w)^\top \A_i \x \Big  )^2\\
&= \frac{(1+\alpha)^2}{k} \sum_{i=1}^k \norm{\w_* - \w}^2 \norm{\A_i \x}^2 \leq  \frac{(1+\alpha)^2}{k} \sum_{i=1}^k \norm{\w_* - \w}^2 \lambda_{\max}( {\A_i \A_i^\top} ) \norm{\x}^2\\
&\leq \frac{(1+\alpha)^2}{k}  \Big ( \sum_{i=1}^k \lambda_{\max}( {\A_i\A_i^\top} ) \Big )\norm{\w_* - \w}^2   \norm{\x}^2.
\end{align*}
\end{proof}

\section{Proof of Lemma \ref{lm1}}\label{lem:gauss}

\begin{proof}
The $\relu$ activation implies that $\alpha = 0$.
Moreover, the normality assumption for the input data yields
$\lambda_1 = \sigma^2$, $\lambda_2=\lambda_3 =1$. 
Standard results about the normal distribution further yield
\begin{equation*}
\mathbb{E}_{\x \sim\mathcal{N}(0,\sigma^2I)}
\left[ \norm{ \x }^k \right]
= \mathbb{E}_{\x \sim\mathcal{N}(0,I)}
\left[ \norm{\sigma \x}^k \right]
= \sigma^k \mathbb{E}_{\x \sim\mathcal{N}(0,I)} \left[ \| \x \|^k \right]
= \sigma^k 2^{k/2}
\frac{\Gamma\left(\frac{n+k}{2}\right)}
{\Gamma\left(\frac{n}{2}\right)}. 
\end{equation*}
Hence we have, 
\begin{equation}\label{eq:beta1}
\beta_1= \beta \mathbb{E}_{\x \sim \mathcal{N}(0,\sigma^2\mathbf{I}_{n\times n})}
\Big[ \norm{ \x } \Big] 
= \sqrt{2}\sigma \beta
\frac{\Gamma\left(\frac{n+1}{2}\right)}{\Gamma\left(\frac{n}{2}\right)}.
\end{equation}
Invoking the above,
the constant ${\rm c}_{\rm trade-off}$ in Theorem \ref{main2} simplifies to Equation \eqref{gauss}.
\end{proof}

\section{Estimating a necessary recursion}\label{app:rec}

\begin{lemma}\label{recurse} 
Suppose we have a sequence of real numbers $\Delta_1, \Delta_2, \ldots$ such that 
\[\Delta_{t+1} \leq (1-\eta' b + \eta'^2 c_1) \Delta_t + \eta'^2c_2 + \eta'c_3\]
for some fixed parameters $b,c_1, c_2, c_3 >0$ such that $\Delta_1 > \frac{c_3}{b}$ and free parameter $\eta' > 0$. Then for,

\[ ~\epsilon'^2 \in \Big ( \frac{c_3}{b} ,\Delta_1 \Big ),\quad ~\eta' = \frac{b}{\gamma c_1}, \quad \gamma > \max \Bigg \{ \frac{b^2}{c_1},  \Bigg ( \frac{\epsilon'^2 + \frac{c_2}{c_1}}{\epsilon'^2   - \frac{c_3}{b}} \Bigg ) \Bigg \}  > 1\]

it follows that $\Delta_{\rm T} \leq \epsilon'^2$ for, 
$${\rm T} = 
{\large \mathcal{O}}\Bigg(\log \Bigg[~ 
\frac{\Delta_1}{\epsilon'^2 - \Big(\frac{\frac{c_2}{c_1}+\gamma \cdot \frac{c_3}{b}}{\gamma - 1}\Big)}~\Bigg]\Bigg) .$$ 
\end{lemma} 

\begin{proof}
Let us define $\alpha = 1-\eta' b + \eta'^2 c_1$ and $\beta = \eta'^2c_2 + \eta'c_3$.
Then by unrolling the recursion we get, 
\[ \Delta_t \leq \alpha \Delta_{t-1} + \beta \leq \alpha (\alpha \Delta_{t-2} + \beta ) + \beta \leq  ...\leq \alpha^{t-1}\Delta_1 + \beta (1+\alpha+\ldots+\alpha^{t-2}). \]

Now suppose that the following are true for $\epsilon'$ as given and for $\alpha ~\& ~\beta$ (evaluated for the range of $\eta'$s as specified in the theorem), 
\begin{itemize}
    \item[ ] {\bf Claim 1 :} $\alpha \in (0,1)$
    \item[ ] {\bf Claim 2 :} $0 < \epsilon'^2(1-\alpha) - \beta$
\end{itemize}

We will soon show that the above claims are true. Now if ${\rm T}$ is such that we have, 

\[ \alpha^{{\rm T}-1}\Delta_1 + \beta (1+\alpha+\ldots+\alpha^{{\rm T}-2}) 
= \alpha^{{\rm T}-1}\Delta_1 + \beta \cdot \frac{1 - \alpha^{\rm T-1}}{1-\alpha}  = \epsilon'^2, \] then $\alpha^{{\rm T} -1} = \frac{\epsilon'^2(1-\alpha) - \beta}{\Delta_1(1-\alpha) - \beta}$. 
Note that {\bf Claim 2} along with with the assumption that $\epsilon'^2 < \Delta_1$ ensures that the numerator and the denominator of the fraction in the RHS are both positive. Thus we can solve for ${\rm T}$ as follows,  

\begin{align}\label{eq:recursionbound}
\nonumber ({\rm T}-1) \log\left(\frac{1}{\alpha}\right) &= \log \bigg[\frac{\Delta_1(1-\alpha)-\beta}{\epsilon'^2  (1-\alpha)-\beta}\bigg] \implies {\rm T}  = {\large \mathcal{O}}\Bigg(\log \Bigg[~ 
\frac{\Delta_1}{\epsilon'^2 - \Big(\frac{\frac{c_2}{c_1}+\gamma \cdot \frac{c_3}{b}}{\gamma - 1}\Big)}~\Bigg]\Bigg).
\end{align} 
In the second equality above we have estimated the expression for ${\rm T}$ after substituting $\eta' = \frac{b}{\gamma c_1}$ in the expressions for $\alpha$ and $\beta$.

~\\
\textbf{Proof of claim 1 : $\alpha \in (0,1)$}

We recall that we have set $\eta' = \frac{b}{\gamma c_1}$. This implies that,
$\alpha  = 1 - \frac{b^2}{c_1}\cdot \Big ( \frac {1}{\gamma} - \frac{1}{\gamma^2} \Big )$. Hence $\alpha > 0$ is ensured by the assumption that $\gamma > \frac{b^2}{c_1}$. And $\alpha < 1$ is ensured by the assumption that $\gamma >1$ 


\textbf{Proof of claim 2 : $0 < \epsilon'^2(1-\alpha) - \beta$}

We note the following, 
\begin{align*}
    - \frac{1}{\epsilon'^2} \cdot \left(\epsilon'^{2}(1-\alpha) -\beta\right) & =\alpha - \Big (1 - \frac{\beta}{\epsilon'^2} \Big ) \\
    & =  1-\frac{b^2}{4c_1} + \Big ( \eta' \sqrt{c_1} - \frac{b}{2\sqrt{c_1}} \Big )^2  - \Big ( 1 - \frac{\beta}{\epsilon'^2}  \Big )\\
    &=\frac{\eta'^2c_2 + \eta'c_3}{\epsilon'^2} + \Big ( \eta' \sqrt{c_1} - \frac{b}{2\sqrt{c_1}} \Big )^2 - \frac{b^2}{4c_1}\\
    &= \frac{\left(\eta'\sqrt{c_2} + \frac{c_3}{2\sqrt{c_2}}\right)^2 -  \frac{c_3^2}{4c_2}}{\epsilon'^2} + \Big ( \eta' \sqrt{c_1} - \frac{b}{2\sqrt{c_1}} \Big )^2 - \frac{b^2}{4c_1}\\
    &= \eta'^2 \Bigg (  \frac{1}{\epsilon'^2} \cdot \left(\sqrt{c_2} + \frac{c_3}{2\eta'\sqrt{ c_2}}\right)^2 + \Big (\sqrt{c_1} - \frac{b}{2\eta ' \sqrt{c_1}} \Big )^2 \\
    & \quad\quad - \frac{1}{\eta'^2} \Bigg[  \frac{b^2}{4c_1} +  \frac{1}{\epsilon'^2}\left( \frac{c_3^2}{4c_2}\right) \Bigg] \Bigg )
\end{align*}

Now we substitute $\eta' = \frac{b}{\gamma c_1}$ for the quantities in the expressions inside the parantheses to get,
\begin{align*}
     - \frac{1}{\epsilon'^2} \cdot \left(\epsilon'^{2}(1-\alpha) -\beta\right) 
    & = \alpha - \Big (1 - \frac{\beta}{\epsilon'^2} \Big ) = \eta'^2 \Bigg ( \frac{1}{\epsilon'^2 } \cdot \left(\sqrt{c_2} + \frac{\gamma c_1c_3}{2b\sqrt{ c_2}}\right)^2 \\
    & \quad \quad \quad +  c_1\cdot\Big (\frac{\gamma}{2} - 1 \Big )^2 - c_1\frac{\gamma^2}{4} - \frac{1}{\epsilon'^2}\cdot \frac{\gamma^2c_1^2c_3^2}{4b^2c_2}  \Bigg )\\
    &= \eta'^2 \Bigg ( \frac{1}{\epsilon'^2} \cdot \left(\sqrt{c_2} + \frac{\gamma c_1c_3}{2b\sqrt{ c_2}}\right)^2 + c_1(1-\gamma) - \frac{1}{\epsilon'^2}\cdot \frac{\gamma^2c_1^2c_3^2}{4b^2c_2} \Bigg )\\
    &= \frac{\eta'^2}{\epsilon'^2} \Bigg (   c_2 + \frac{\gamma c_1c_3}{b} - \epsilon'^2  c_1(\gamma -1) \Bigg )\\
    &= \frac{\eta'^2c_1}{\epsilon'^2} \Bigg (   (\epsilon'^2   + \frac{c_2}{c_1})  - \gamma \cdot \left (\epsilon'^2 - \frac{c_3}{b} \right )  \Bigg ) 
\end{align*}    

Therefore, $-\frac{1}{\epsilon'^2}\left(\epsilon'^{2}(1-\alpha) - \beta\right) < 0$ since by assumption $\epsilon'^2   > \frac{c_3}{b},~\text{ and }~ \gamma > \frac{\left(\epsilon'^2 + \frac{c_2}{c_1}\right)}{\epsilon'^2   - \frac{c_3}{b} }.$


\end{proof}




\newpage
\section{Neuro-Tron versus SGD comparisons for different input data distributions}
\label{app:tron_vs_sgd}

\begin{figure}[htbp!]
\centering
\begin{subfigure}{.5\textwidth}
  \centering
  \includegraphics[width=1.\linewidth]{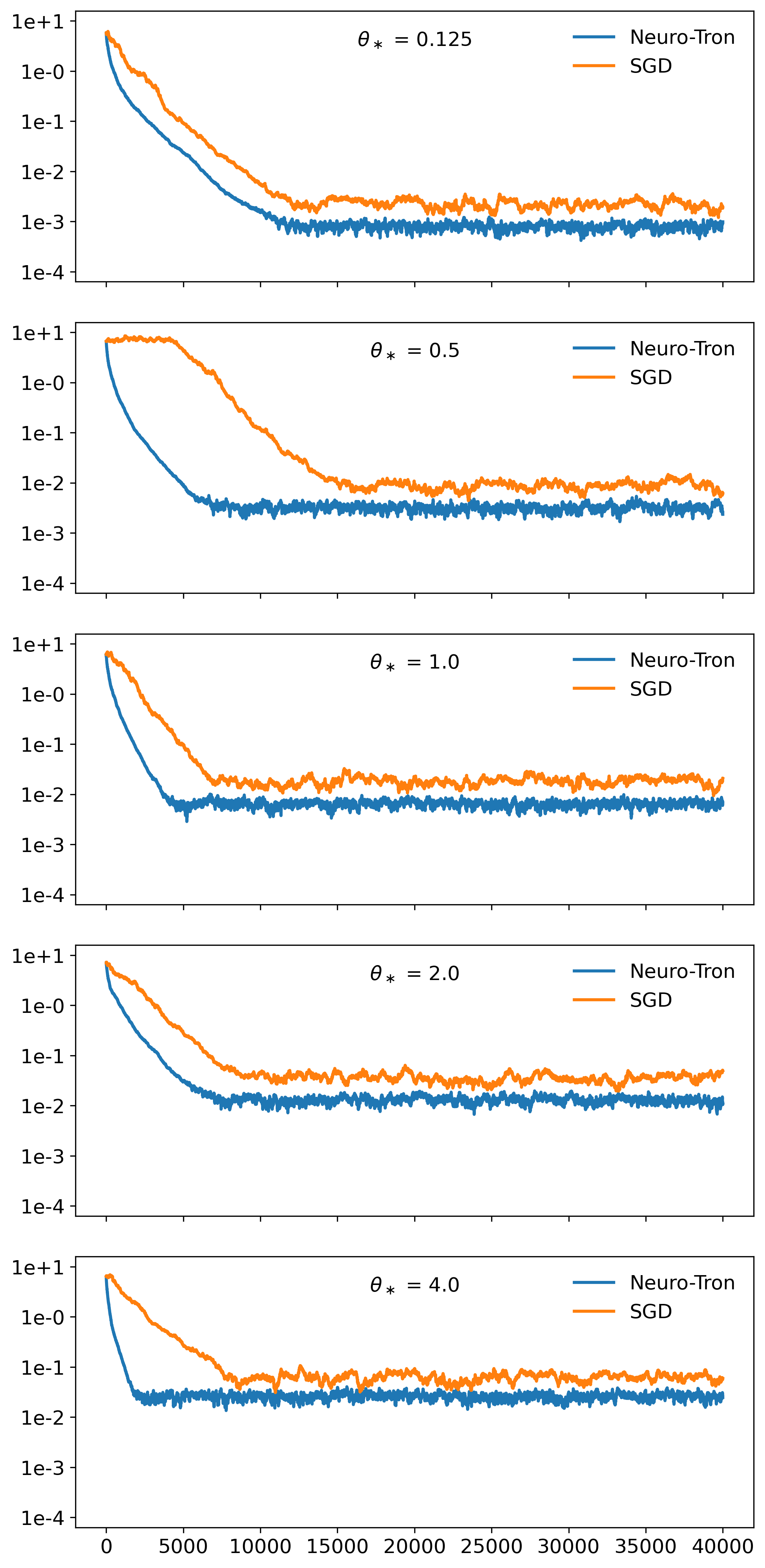}
  \caption{NeuroTron vs SGD per $\theta_{\star}$ value.}
  \label{fig:tron_vs_sgd_theta_laplace}
\end{subfigure}%
\begin{subfigure}{.5\textwidth}
  \centering
  \includegraphics[width=1.\linewidth]{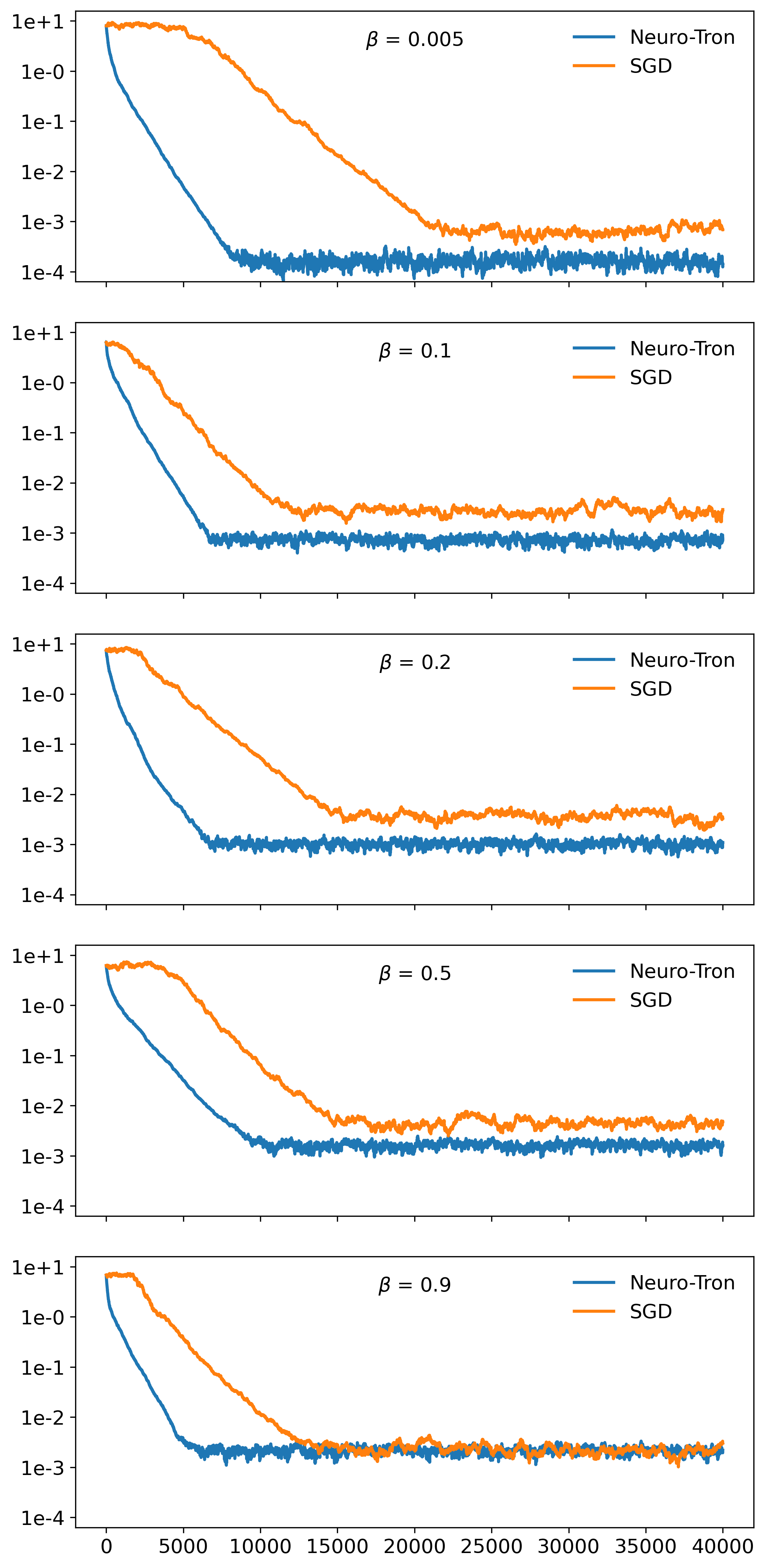}
  \caption{NeuroTron vs SGD per $\beta$ value.}
  \label{fig:tron_vs_sgd_beta_laplace}
\end{subfigure}
\caption{Simulation-based comparison
between Algorithm \ref{tronalgo} (Neuro-Tron) and SGD.
Input data are sampled from $\mbox{Laplace}(\mu=0, b=2)$.
(a): Parameter recovery errors
for different adversarial noise bounds $\theta_{\star}$.
(b): Parameter recovery errors
for different probabilities $\beta$ of adversarial attack.
}
\label{fig:tron_vs_sgd_laplace}
\end{figure}

\clearpage 
\begin{figure}[htbp!]
\centering
\begin{subfigure}{.5\textwidth}
  \centering
  \includegraphics[width=1.\linewidth]{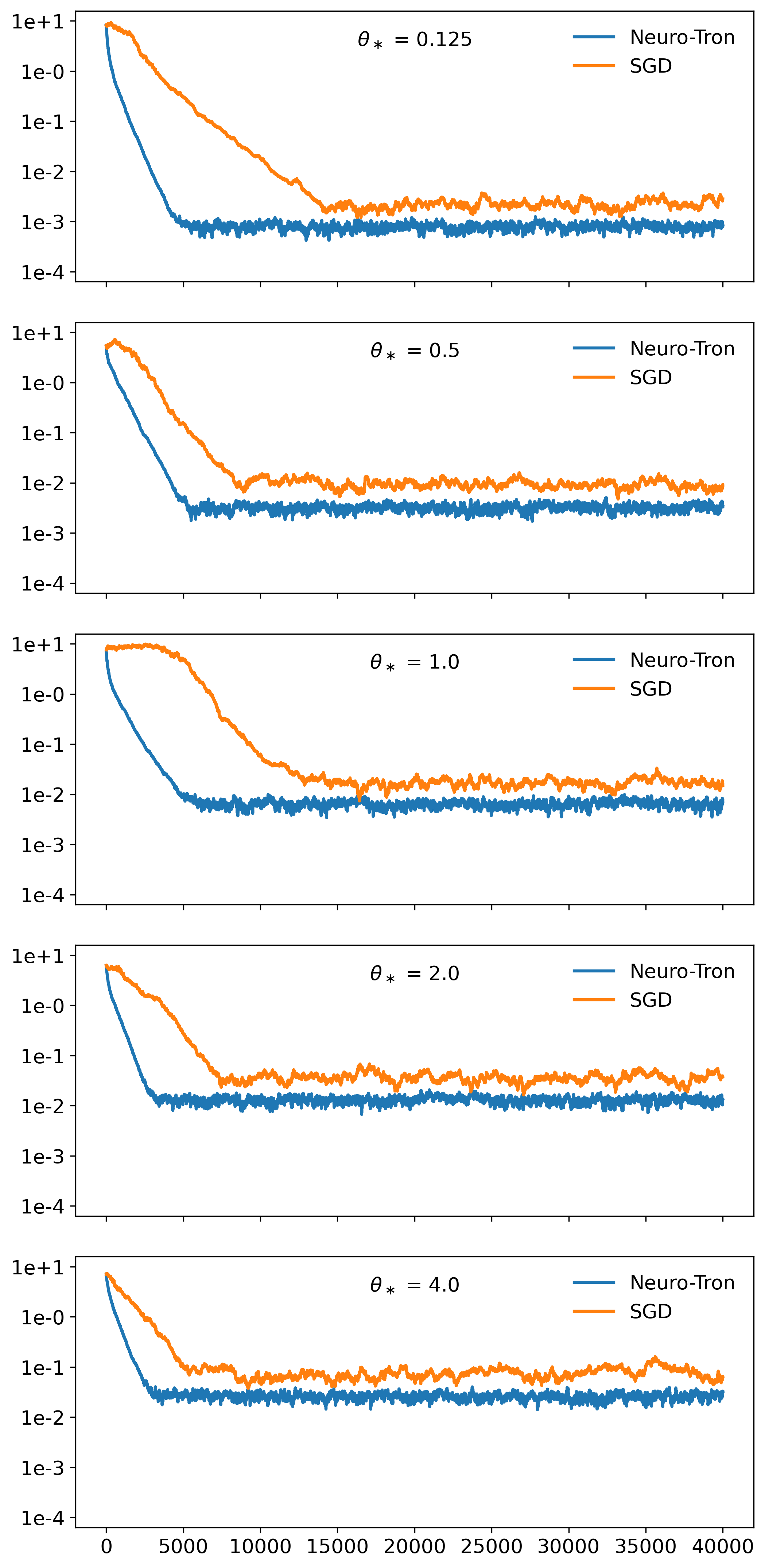}
  \caption{NeuroTron vs SGD per $\theta_{\star}$ value.}
  \label{fig:tron_vs_sgd_theta_normal_large_var}
\end{subfigure}%
\begin{subfigure}{.5\textwidth}
  \centering
  \includegraphics[width=1.\linewidth]{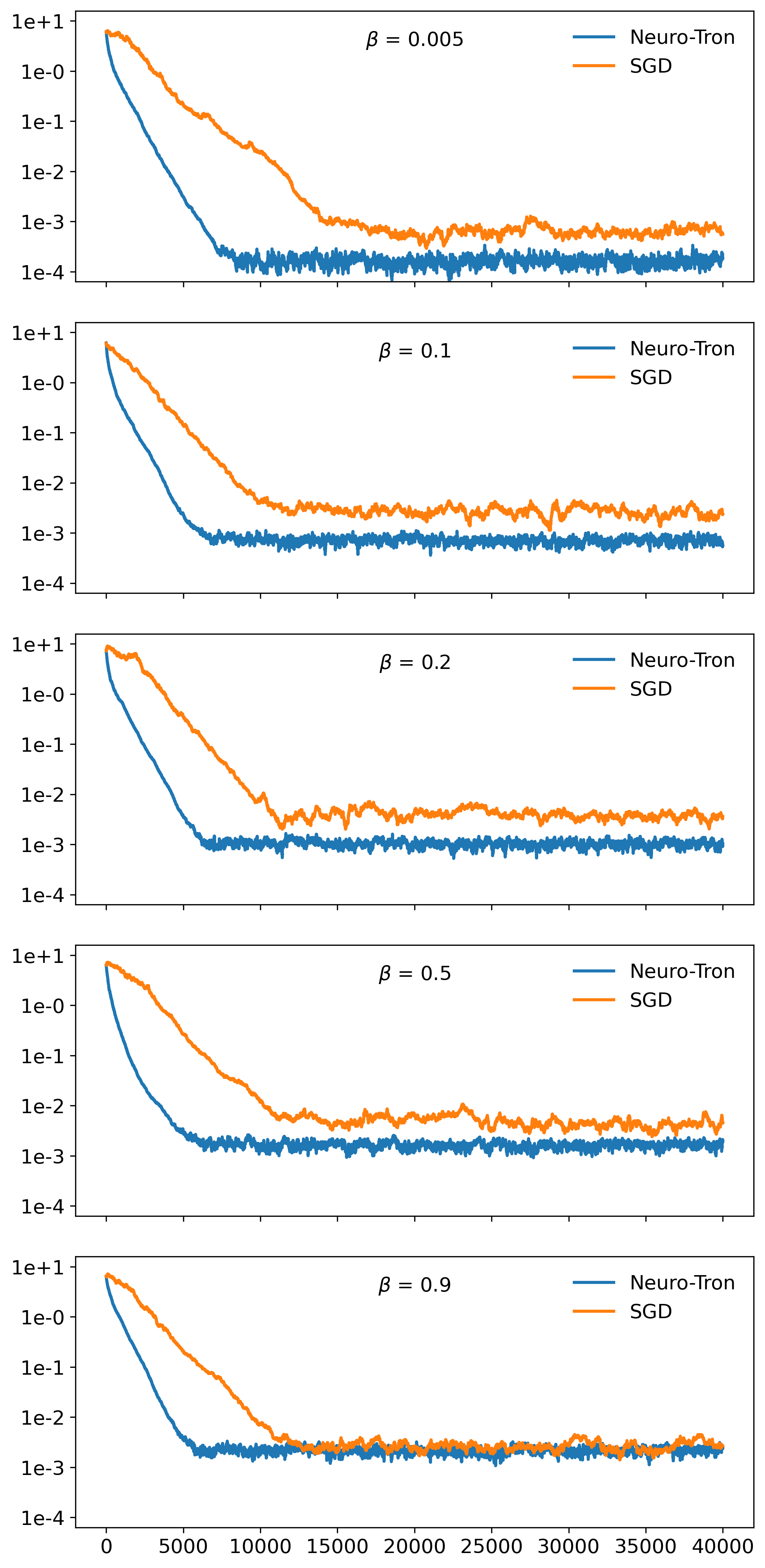}
  \caption{NeuroTron vs SGD per $\beta$ value.}
  \label{fig:tron_vs_sgd_beta_normal_large_var}
\end{subfigure}
\caption{Simulation-based comparison
between Algorithm \ref{tronalgo} (Neuro-Tron) and SGD.
Input data are sampled from $\mathcal{N}(\mu=0, \sigma^2=9)$.
(a): Parameter recovery errors
for different adversarial noise bounds $\theta_{\star}$.
(b): Parameter recovery errors
for different probabilities $\beta$ of adversarial attack.
}
\label{fig:tron_vs_sgd_normal_large_var}
\end{figure}

\clearpage 
\section{Neuro-Tron error under no attack}
\label{app:tron_no_attack}

\begin{figure}[htbp!]
\centering
\includegraphics[width=1.\linewidth]{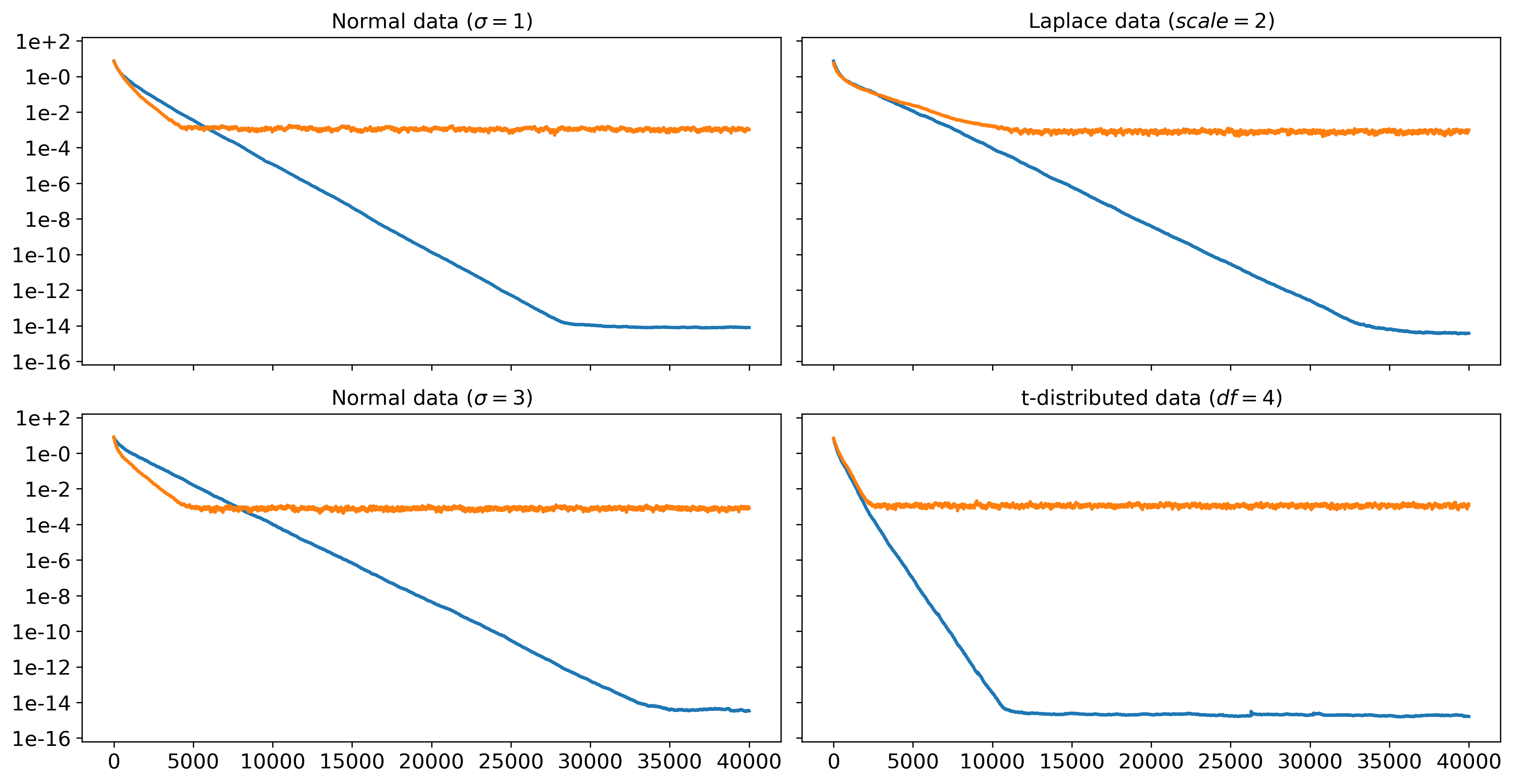}
\caption{Performance demonstration of Algorithm \ref{tronalgo} (Neuro-Tron)
in the absence of data-poisoning attack ($\theta_{*}=0$).
Parameter recovery errors are shown for different input data distributions.
Blue and orange lines correspond to
noise bounds $\theta_{*}=0$ (no attack)
and $\theta_{*}=0.125$ (attack of relatively small magnitude).}
\label{fig:no_attack}
\end{figure}


\newpage
\section{Demonstrating the utility of heavier outer layer weights}
\label{ssc:highqsim}

\begin{figure}[htbp!]
\centering
\begin{subfigure}{.5\textwidth}
  \centering
  \includegraphics[width=1.\linewidth]{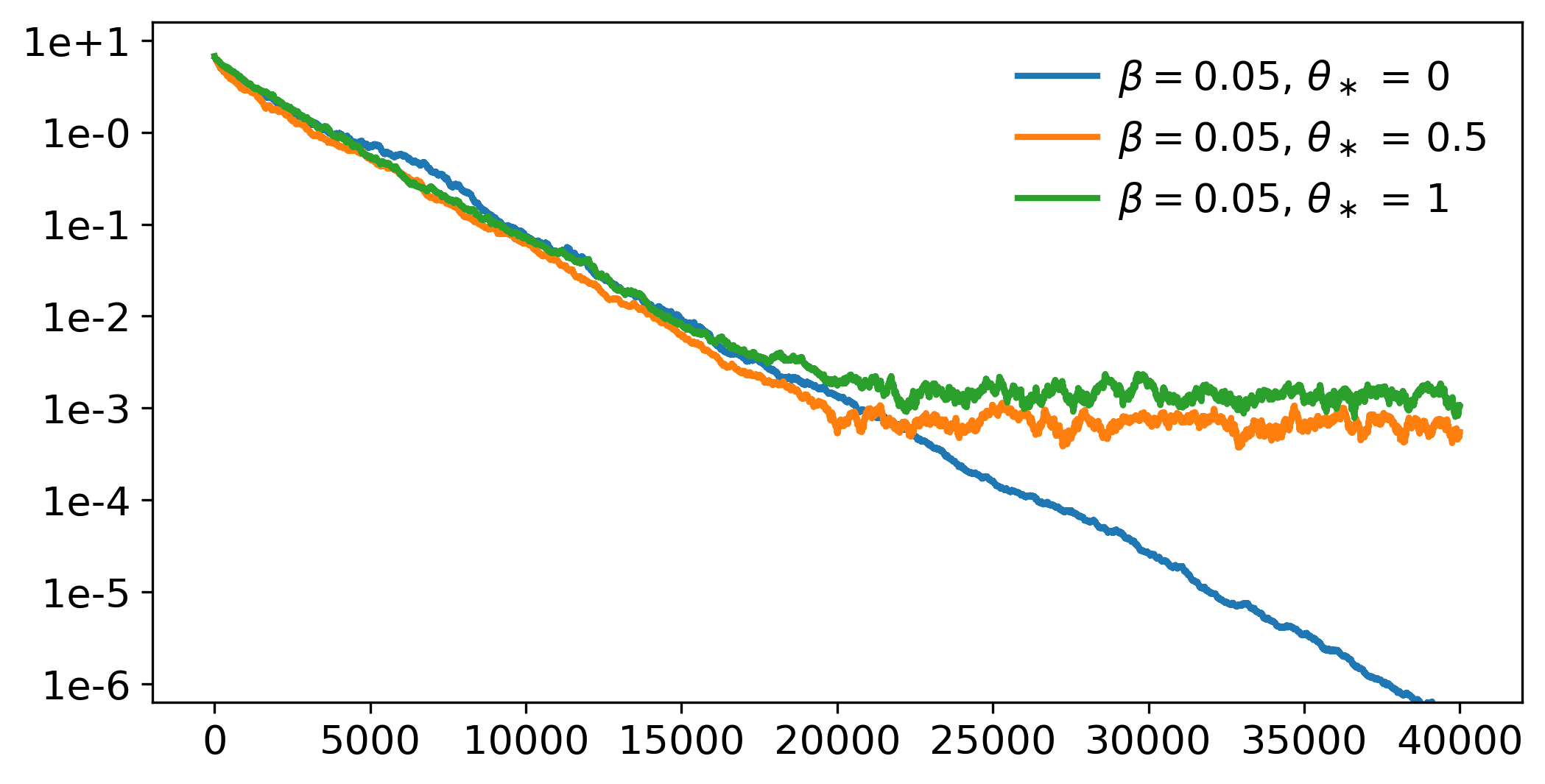}
  \caption{$q_i=1$.}
  \label{fig:highq1}
\end{subfigure}%
\begin{subfigure}{.5\textwidth}
  \centering
  \includegraphics[width=1.\linewidth]{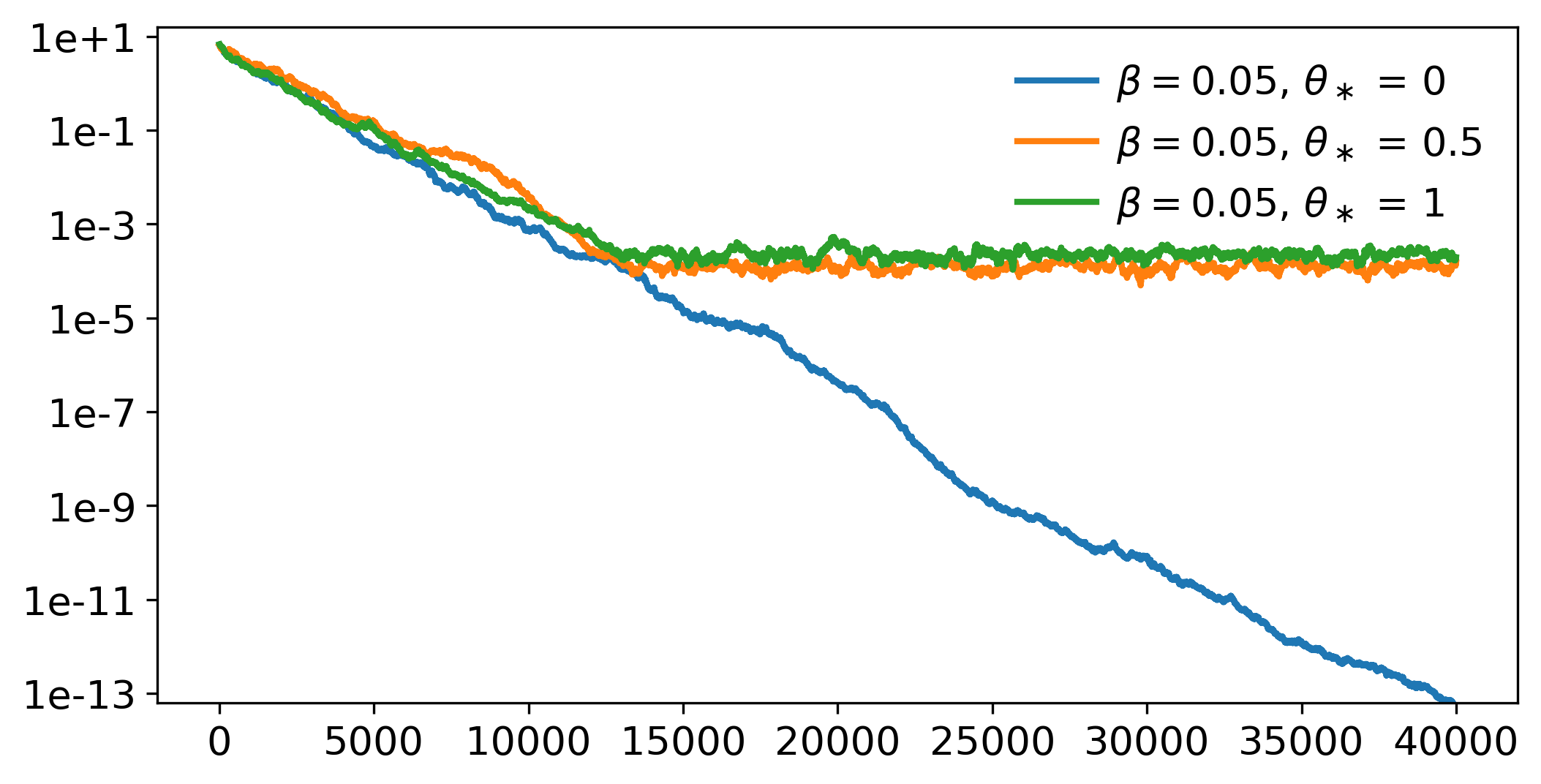}
  \caption{$q_i=10$.}
  \label{fig:highq10}
\end{subfigure}
\caption{Simulation-based validation of the advantage when heavier outer layer weights are used, as described in Section \ref{ssc:highq}.
The experiments are run for (a) $q_i=1$ and (b) $q_i=10$ for all $i$,
with $(\beta_1,\theta_*)$ taking values $(0.05,0), (0.05,0.5)$ and $(0.05,1)$.}
\label{fig:highq}
\end{figure}

Based on the setup of Section \ref{ssc:highq} for Algorithm~\ref{tronalgo},
here we perform a simulation-based comparison between
$q_i=1$ and $q_i=10$ for all $i$.
We run both experiments with the same $\M$, $A_i$ and other hyperparameters.
The simulation adheres to the overall experimental setup of Section~\ref{sec:sim}.
Figure \ref{fig:highq} demonstrates
the uniform advantage of having heavier outer layer weights
in terms of achieving better accuracy (lower parameter recovery error).


\end{document}